%% file: Adaptive_active_learning_arXiv.tex
\documentclass{article}

\usepackage[numbers]{natbib}

\usepackage[preprint]{nips_2018}

\usepackage[utf8]{inputenc} % allow utf-8 input
\usepackage[T1]{fontenc}    % use 8-bit T1 fonts
\usepackage{hyperref}       % hyperlinks
\usepackage{url}            % simple URL typesetting
\usepackage{booktabs}       % professional-quality tables
\usepackage{amsfonts}       % blackboard math symbols
\usepackage{nicefrac}       % compact symbols for 1/2, etc.
\usepackage{microtype}      % microtypography
\usepackage{xcolor}

\usepackage{mathrsfs}
\usepackage{amssymb}
\usepackage{bm}
\usepackage{amsmath,amsthm }

\DeclareMathOperator*{\argmin}{argmin}
\newcommand{\nn}{\nonumber}

\newcommand{\ba}{\text{\boldmath{$a$}}}

\usepackage{graphicx} % more modern
\usepackage{subfigure}
\usepackage{algorithm}
\usepackage{algorithmic}

\newtheorem{thm}{Theorem}%[section]
\newtheorem{lem}{Lemma}
\newtheorem{rem}{Remark}

\newtheorem{assum}{Assumption}

\title{Active and Adaptive Sequential learning}

\author{
    Yuheng Bu \thanks{Equal contribution}\ \ \thanks{ECE Department and the
Coordinated Science Laboratory, University of Illinois at Urbana-Champaign,
Urbana, IL, USA. Email: \{bu3, vvv\}@illinois.edu}\\
  \And
  Jiaxun Lu \footnotemark[1]\ \  \thanks{EE Department, Tsinghua University,
Beijing, China. Email: lujx14@mails.tsinghua.edu.cn} \\
  \And
  Venugopal V. Veeravalli  \footnotemark[2]\\
}

\begin{document}

\maketitle

\begin{abstract}
A framework is introduced for actively and adaptively solving a sequence of machine learning problems, which are changing in bounded manner from one time step to the next. An algorithm is developed that actively queries the labels of the most informative samples from an unlabeled data pool, and that adapts to the change by utilizing the information acquired in the previous steps. Our analysis shows that the proposed active learning algorithm based on stochastic gradient descent achieves a near-optimal excess risk performance for maximum likelihood estimation. Furthermore, an estimator of the change in the learning problems using the active learning samples is constructed, which provides an adaptive sample size selection rule that guarantees the excess risk is bounded for sufficiently large number of time steps. Experiments with synthetic and real data are presented to validate our algorithm and theoretical results.
\end{abstract}
%Adaptively selecting the number of samples
%and  sequentially applying optimization algorithm to minimize the loss function, such as stochastic gradient descent (SGD).
%along with analysis to show that eventually this estimate upper bounds the change in minimizers.
%An estimate of a bound on the minimizers’ change, combined with properties of the chosen optimization algorithm, is used to select the number of samples needed to
%meet the desired tracking criterion.
%When the upper bound of the change in the minimizers is known, we provide theoretical guarantees for our active and adaptive learning algorithm, which ensures the excess risk of each learning task is always bounded.

\section{Introduction}
\label{Sec:Introduction}

Machine learning problems that vary in a bounded manner over time naturally arise in many applications. For example, in personalized recommendation systems \cite{elahi2016survey,rubens2015active}, the preferences of users might change with fashion trends. Since acquiring new training samples from users can be expensive in practice, a recommendation system needs to update the machine learning model and adapt to this change using as few new samples as possible.

%Other applications including personalized spam detection \cite{bickel2006ecml}, channel estimation and parameter tracking.

In such problems, we are given a large set of unlabeled samples, and the learning tasks are solved by minimizing the expected value of an appropriate loss function on this unlabeled data pool at each time $t$. To capture the idea that the sequence of learning problems is changing in a bounded manner, we assume the following bound holds
\begin{equation}\label{equ:BoundedChangeRate}
\|\theta_{t}^*-\theta_{t-1}^*\|_{2} \leq \rho,\quad \forall t\ge 2,
\end{equation}
where $\theta_{t}^*$ is the true minimizer of the loss functions at time $t$, and $\rho$ is a finite upper bound on the change of minimizers, which needs to be estimated in practice.

%$\ell(z,\theta)$ at each time $t$:
%\begin{equation}\label{equ:intro}
%  \min_{\theta\in \Theta} \Big\{L_{U_t}(\theta)\triangleq \mathbb{E}_{Z_t\sim p_t}[\ell(Z_t,\theta)] \Big\}, \quad \forall t\ge 1,
%\end{equation}
%where $p_t$ denotes the underlying (unknown) probabilistic distribution for the data $z_t$, which is time-varying. The data $z_t=\{x_t,y_t\}$ corresponds to the \{predictor, response\} pairs in a regression problem and corresponds to the \{feature, label\} pairs in a classification problem, and $\theta$ parameterizes the regression model or the classifier in these two problems, respectively.

%and we need to interactively query the labels of the most informative samples to adapt to changes.

To tackle this sequential learning problem, we propose an \emph{active} and \emph{adaptive} algorithm to learn the approximate minimizers $\hat{\theta}_t$ of the loss function. At each time $t$, our algorithm actively queries the labels of $K_t$ samples from the unlabeled data pool, with a well-designed active sampling distribution, which is adaptive to the change in the minimizers by utilizing the information acquired in the previous steps. In particular, we adaptively select $K_t$ and construct $\hat{\theta}_t$ such that the excess risk \cite{mohri2012foundations} is bounded at each time $t$.

The challenges of this active and adaptive sequential learning problem arise in three aspects:  1) we need to determine which samples are more informative for solving the task at the current time step based on the information acquired in the previous time steps to conduct active learning; 2) to achieve a desired bounded excess risk with as few new samples as possible, we need to understand the tradeoff between the solution accuracy and the adaptively determined sample complexity $K_t$; 3) the change in the minimizers $\rho$ is unknown and we need to estimate it.

%Under this model, we sequentially learn approximate
%minimizers $\hat{\theta}_t$ of function $L_t(\theta)$ using $K_t$ training samples
%$\{z_{k,t}\}_{k=1}^{K_t} \sim p_t$ at time $t$ by applying an optimization algorithm
%such as stochastic gradient descent (SGD) starting from the previous approximate
%minimizer $\hat{\theta}_{t-1}$. Unlike the previous work on adaptive sequential learning \citet{wilson2018adaptive}, where the training samples are i.i.d. drawn from $p_t$ passively, since the change in the minimizer is bounded by \eqref{equ:BoundedChangeRate}, we study the case where training samples used at time $t$ are chosen actively based on the knowledge we learnt in previous step $t-1$, i.e., active adaptive sequential learning.

%Our goal is to \emph{adaptively} determine the number of samples $K_t$ required to achieve a desired excess risk for each $t$, and \emph{actively} choose $K_t$ samples from the data pool and query their responses (labels) based on the previous estimation to further reduce the excess risk. Moreover, $\rho$ is unknown in practice, we also need to construct an estimates of $\rho$, which can be used to determine the number of samples $K_t$ required to achieve a target excess risk.

Our contributions in this paper can be summarized as follows. We propose an active and adaptive learning framework with theoretical guarantees to solve a sequence of learning problems, which ensures a bounded excess risk for each individual learning task when $t$ is sufficiently large. We construct a new estimator of the change in the minimizers $\hat{\rho}_t$ with active learning samples and show that this estimate upper bounds the true parameter $\rho$ almost surely. We test our approaches on a synthetic regression problem, and further apply it to a recommendation system that tracks changes in preferences of customers. Our experiments demonstrate that our algorithm achieves a better performance compared to the other baseline algorithms in these scenarios.

\subsection{Related Work}

Our active and adaptive learning problem has relations with \emph{multi-task learning} (MTL) and \emph{transfer learning}. In multi-task learning, the goal is to learn several tasks simultaneously as in \cite{agarwal2010learning,evgeniou2004regularized,zhang2012convex} by exploiting the similarities between the tasks. In transfer learning, prior knowledge from one source task is transferred to another target task either with or without additional training data \cite{pan2010survey}. Multi-task learning could be applied to solve our problem by running a MTL algorithm at each time, while remembering all prior tasks. However, this approach incurs a heavy memory and computational burden. Transfer
learning lacks the sequential nature of our problem, and there is no active learning component in both works. For multi-task and transfer learning, there are theoretical guarantees on regret for some algorithms \cite{agarwal2008matrix}, while we provide an excess risk guarantee for each individual task.

In \emph{concept drift} problem, stream of incoming data that changes over time is observed, and we try to predict some properties of each piece of data as it arrives. After prediction, a loss is revealed as the feedback in \cite{towfic2013distributed}. Some approaches for concept drift use iterative algorithms such as stochastic gradient descent, but without specific models on how the data changes, there is no theoretical guarantees for these algorithms.

Our work is of course related to active learning \cite{dasgupta2006coarse,balcan2006agnostic}, in which a learning algorithm is able to interactively query the labels of samples from an unlabeled data pool to achieve better performance. A standard approach to active learning is to select the unlabeled samples by optimizing specific statistics of these samples \cite{cornell2011experiments}.  For example, with the goal of minimizing the expected excess risk in maximum likelihood estimation, the authors of \cite{chaudhuri2015convergence,sourati2017asymptotic} propose a two-stage algorithm based on Fisher information ratio to select the most informative samples, and show that it is optimal in terms of the convergence rate. We apply similar algorithms in our problem, but the first stage of estimating the Fisher information using labeled samples to conduct active learning can be skipped by exploiting the bounded nature of the change, and utilizing information obtained in previous time steps.

%Active learning has wide applications in which unlabeled data is abundant but manually labeling is expensive, for example, recommendation systems \cite{elahi2016survey,rubens2015active}.

Our approach is closely related to prior work on adaptive sequential learning \cite{wilson2018adaptive,wilson2016adaptive}, where the training samples are drawn passively and the adaptation is only in the selection of the number of training samples $K_t$ at each time step. %In this work, we study the case where training samples used at time $t$ are chosen actively and adaptively based on the information acquired in the previous time steps.

The rest of the paper is organized as follows. In Section \ref{Sec:ProblemFormulation}, we describe the problem setting considered. In Section \ref{Sec:Algorithm}, we present our active and adaptive learning algorithm. In Section \ref{Sec:Theorem}, we provide the theoretical analysis which motivates the proposed algorithm. In Section \ref{Sec:experiments}, we test our algorithm on synthetic and real data. Finally, in Section \ref{Sec:conclusion}, we provide some concluding remarks.

%In Section \ref{Sec:SeqActLearning_unknown}, we construct an estimator of $\rho$ for our active learning algorithm when $\rho$ is unknown, along with the analysis to show that eventually the estimate upper bounds the change in minimizers.

\section{Problem Setting}
\label{Sec:ProblemFormulation}

Throughout this paper, we use lower case letters to denote scalars and vectors, and use upper case letters to denote random variables and matrices. All logarithms are the natural ones. We use $I$ to denote an identity matrix of appropriate size. We use the superscript $(\cdot)^\top$ to denote the transpose of a vector or a matrix, and use $\mathrm{Tr}(A)$ to denote the trace of a square matrix $A$. We denote $\|x\|_A=\sqrt{x^\top A x}$ for a vector $x$ and a matrix $A$ of appropriate dimensions.

We consider the active and adaptive sequential learning problem in the maximum likelihood estimation (MLE) setting. At each time $t$, we are given a pool $\mathcal{S}_{t}=\{ x_{1,t},\cdots,x_{N,t}\}$ of $N_t$ unlabeled samples drawn from some instance space $\mathcal{X}$. We have the ability to interactively query the labels of $K_t$ of these samples from a label space $\mathcal{Y}$. In addition, we are given a parameterized family of distribution models $\mathcal{M}=\{ p(y|x,\theta_{t}), \theta_{t}\in\Theta\}$, where $\Theta\subseteq \mathbb{R}^d$. We assume that there exists an unknown parameter $\theta_{t}^*\in\Theta$ such that the label $y_{t}$ of $x_{t} \in \mathcal{S}_{t}$ is actually generated from the distribution $p(y_t|x_{t},\theta_{t}^*)$.

For any $x\in \mathcal{X}$, $y\in \mathcal{Y}$ and $\theta\in \Theta$, we let the loss function be the negative log-likelihood with parameter $\theta$, i.e.,
\begin{equation}
\ell(y|x,\theta) \triangleq -\log p(y|x,\theta),\quad  p(y|x,\theta)\in \mathcal{M}.
\end{equation}
Then, the expected loss function over the uniform distribution on the data pool $\mathcal{S}_{t}$ can be written as
\begin{equation}\label{equ:ExpectedLogLikelihood}
L_{U_t}(\theta) \triangleq \mathbb{E}_{X \sim U_t,Y \sim p(Y|X,\theta_t^*)}[ \ell(Y|X,\theta) ],
\end{equation}
where we use $U_t$ to denote the uniform distribution over the samples in $\mathcal{S}_{t}$. It can be seen that the minimizer of $L_{U_t}(\theta)$ is the true parameter $\theta_t^*$. As mentioned in \eqref{equ:BoundedChangeRate}, we assume that $\theta_{t}^*$ is changing at a bounded but unknown rate, i.e., $\|\theta_{t}^*-\theta_{t-1}^*\|_{2} \leq \rho$, for $t\ge 2$.

The quality of our approximate minimizers $\hat{\theta}_t$ are evaluated through a \emph{mean tracking criterion}, which means that the excess risk of $\hat{\theta}_t$ is bounded at each time step $t$, i.e.,
\begin{equation}\label{equ:excess_risk_criterion}
  \mathbb{E}[L_{U_t}(\hat{\theta}_t) - L_{U_t}({\theta}_t^*)]\le \varepsilon.
\end{equation}
Thus, our goal is to actively and adaptively select the smallest number of samples $K_t$ in $\mathcal{S}_{t}$ to query labels, and sequentially construct an estimate of $\hat{\theta}_t$ satisfying the above mean tracking criterion for each time step $t$. Note that it is allowed to query the label of the same sample multiple times.

Let $\Gamma_t$ be an arbitrary sampling distribution on $\mathcal{S}_t$. Then, the following MLE using $\Gamma_t$
\begin{equation}\label{equ:MLE}
\hat{\theta}_{\Gamma_{t}} \triangleq \argmin_{\theta \in\Theta} \frac{1}{K_t} \sum_{k=1}^{K_t} \ell(Y_{k,t}|X_{k,t},\theta),
\end{equation}
can be viewed as an empirical risk minimizer (ERM) of \eqref{equ:ExpectedLogLikelihood}, where $X_{k,t} \sim \Gamma_{t}$, $Y_{k,t} \sim p(Y|X_{k,t},\theta_t^*)$.

To ensure that our algorithm works correctly, we require the following assumption on the Hessian matrix of  $\ell(y|x,\theta)$, which determines the Fisher information matrix.
\begin{assum}\label{cond:IndependCond}
For any $x\in\mathcal{X}$, $y\in\mathcal{Y}$, $\theta\in\Theta$, $H(x,\theta)\triangleq\frac{\partial^2 \ell(y|x, \theta)}{\partial \theta^2}$ is a function of only $x$ and $\theta$ and does not depend on $y$.
\end{assum}
Assumption \ref{cond:IndependCond} holds for many practical models, such as generalized linear model, logistic regression and conditional random fields \cite{chaudhuri2015convergence}.  Moreover, for $\theta\in \Theta$, we denote $I_{\Gamma_t}(\theta) \triangleq \mathbb{E}_{X\sim\Gamma_t} [H(X,\theta)]$  as the Fisher information matrix under sampling distribution $\Gamma_t$.

%The convergence rate of the excess risk using MLE is well-established in the literature, as shown in \cite{chaudhuri2015convergence} and Lemma \ref{lem:BoundsWithArbitrarySamplingDistri} in this paper,
%\begin{equation}\label{equ:ERM_rate}
%  \tau_t^2 \triangleq \lim_{K_t \to \infty} \frac{\mathbb{E}[L_{U_t}(\hat{\theta}_{\Gamma_t}) - L_{U_t}(\theta_t^*)]}{1/K_t} = \frac{1}{2}\mathrm{Tr}\big( I_{\Gamma_t}^{-1}(\theta_t^*) I_{U_t}(\theta_t^*) \big)
%\end{equation}

%

%\section{Main Results}
%\label{Sec:main_results}

%In this section, we first introduce our adaptive and active sequential learning algorithm, and present the theoretical performance guarantees of the proposed algorithm.
%
%%, which is proved to be achievable in practical applications.

\section{Algorithm}
\label{Sec:Algorithm}

The main idea of our algorithm is to adaptively choose the number of samples $K_t$ based on the estimated change in the minimizers $\hat{\rho}_{t-1}$ such that the mean tracking criterion in \eqref{equ:excess_risk_criterion} is satisfied, then actively query the labels of these $K_t$ samples with a well-designed sampling distribution $\Gamma_t$, and finally perform MLE in \eqref{equ:MLE} using a stochastic gradient descent (SGD) algorithm over the labeled samples. By executing this algorithm iteratively, we can sequentially learn $\hat{\theta}_{t}$ over all the considered time steps. The algorithm is formally presented in Algorithm \ref{alg:AdapActSeqLearning}.

%For each time $t$, our algorithm contains five steps:
%\begin{enumerate}
%  \item Constructing active sampling distribution $\Gamma_t$ using previous estimation $\hat{\theta}_{t-1}$;
%  \item Adaptively selecting the number of samples $K_t$ based on the estimated change in the minimizer $\hat{\rho}_{t-1}$ such that the mean tracking criterion in \eqref{equ:excess_risk_criterion} is satisfied;
%  \item Querying labels of $K_t$ samples generated by $\Gamma_t$ and performing ML estimation using SGD algorithm satisfying Assumption \ref{assum:SGD_bound} over them;
%  \item Estimating the change in the minimizer $\hat{\rho}_t$ with the $\hat{\theta}_{t}$ and $\hat{\theta}_{t-1}$.
%\end{enumerate}

To ensure a good performance with limited querying samples, it is essential to construct $\Gamma_t$  carefully. Motivated by Lemma \ref{lem:BoundsWithArbitrarySamplingDistri} in Section \ref{Sec:active}, the convergence rate of the excess risk for ERM using $K_t$ samples from $\Gamma_t$ is $\mathrm{Tr}( I_{\Gamma_t}^{-1}(\theta_t^*)I_{U_t}(\theta_t^*) )/K_t$. Thus, the optimal sampling distribution $\Gamma_t^*$ should be the one that minimizes $\mathrm{Tr}( I_{\Gamma_t}^{-1}(\theta_t^*) I_{U_t}(\theta_t^*) )$, which relies on the unknown parameter $\theta_t^*$. Based on the bounded nature of the change in \eqref{equ:BoundedChangeRate}, we solve this problem by approximating $\theta_{t}^*$ with $\hat{\theta}_{t-1}$  and generate the sampling distribution $\hat{\Gamma}_t^*$ by minimizing $\mathrm{Tr}( I_{\Gamma_t}^{-1}(\hat{\theta}_{t-1}) I_{U_t}(\hat{\theta}_{t-1}) )$ (Step 1).

%In next subsection, we show that this approximation is good enough given the slowly change condition as in \eqref{equ:BoundedChangeRate}.

\begin{algorithm}[h]
   \caption{Active and Adaptive Sequential Learning}
   \label{alg:AdapActSeqLearning}
\begin{algorithmic}
   \STATE {\bfseries Input:} Sample pool $\mathcal{S}_{t}=\{ x_{1,t},\cdots,x_{N,t}\}$, the previous estimation $\hat{\theta}_{t-1}$, $\hat{\rho}_{t-1}$ and the desired mean tracking accuracy $\varepsilon$.
   \STATE 1: Solve the following semidefinite programming problem (see Section \ref{Sec:active})
   \begin{flalign*}
  \hat{\Gamma}_t^* = \argmin_{\Gamma_t \in \mathbb{R}^{N_t}}\,\,\, {\mathrm{Tr}}[I_{\Gamma_t}^{-1}(\hat{\theta}_{t-1})I_{U_{t}}(\hat{\theta}_{t-1})] \quad \textrm{s.t.}\,\,\,&
  \begin{cases}
  I_{\Gamma_t}(\hat{\theta}_{t-1}) = \sum_{i=1}^{N_t} \Gamma_{i,t} H(x_{i,t},\hat{\theta}_{t-1} ),\\
  \sum_{i=1}^{N_t} \Gamma_{i,t} = 1,\, \Gamma_{i,t}\in[0,1].
  \end{cases}
  \end{flalign*}
   \STATE 2: Choose $K_t^*$ based on $\hat{\rho}_{t-1}$ such that it is the minimum number of samples required to meet the mean tracking criterion (see Section \ref{Sec:mean_tracking}).
%   Determine the number of samples as follows (refer Theorem \ref{thm:mean_tracking})
%   \begin{equation*}
%   K_t = \min\Big\{K\ge 1 \Big| b\Big(d, \sqrt{\frac{2\varepsilon}{m}}+\hat{\rho}_{t-1}, K \Big) \le \varepsilon \Big\}.
%   \end{equation*}
   \STATE 3: Generate $K^*_t$ samples using the distribution $\bar{\Gamma}_t=\alpha_{t}\hat{\Gamma}_t^*+(1-\alpha_{t})U_t$ on unlabeled data pool $\mathcal{S}_t$, where $\alpha_{t}\in(0,1)$. Query their labels and get the labeled set $\mathcal{S}_{t}'=\{ (x_{k,t},y_{k,t})\}_{k=1}^{K_t^*}$.
   \STATE 4: Solve the MLE using labeled set $S_{t}'$ with a SGD algorithm initialized at $\hat\theta_{t-1}$,
   \begin{equation*}
   \hat{\theta}_{t} = \argmin_{\theta_{t}\in\Theta} \sum_{(x_{k,t},y_{k,t})\in \mathcal{S}_{t}'} \ell(y_{k,t}|x_{k,t},\theta_{t}).
   \end{equation*}
   \STATE 5: Update the estimate of $\hat{\rho}_{t}$ using estimator defined in Section \ref{Sec:estimator} for $\forall t \geq 2$.
   \STATE {\bfseries Output:} $\hat{\theta}_{t}$, $\hat{\rho}_{t}$.
\end{algorithmic}
\end{algorithm}

Then, as shown in Section \ref{Sec:mean_tracking}, we use the minimum number of samples $K_t^*$ such that the mean tracking criterion is satisfied, and actively draw samples from $\bar{\Gamma}_t$ to estimate $\hat{\theta}_{t}$  (Steps 2-4). Note that the distribution $\hat{\Gamma}_t^*$ is modified slightly to $\bar{\Gamma}_t$ in Step 3 to ensure it still has the full support of $\mathcal{S}_t$.

Finally, based on the current and previous estimation $\hat{\theta}_{t}$ and $\hat{\theta}_{t-1}$, we update the estimate of the bounded change rate $\hat{\rho}_t$ by the estimator proposed in Section \ref{Sec:estimator}.

It is easy to see that the active nature of Algorithm \ref{alg:AdapActSeqLearning} comes from the active sampling distribution, which is constructed by minimizing the Fisher information ratio as in Step 1. But the adaptivity of Algorithm \ref{alg:AdapActSeqLearning} is more complex and results from the following three aspects: 1) The sampling distribution is adaptive to the bounded change through the replacement of $\theta^*_t$ with $\hat{\theta}_{t-1}$ in Step 1; 2) The sample size selection rule is adaptive through the selection of the minimum  number of samples required in Step 2; 3) The SGD is adaptive through the initialization by  $\hat\theta_{t-1}$ in Step 4. %, which speeds up the convergence of SGD algorithm.

%Here, we present the pseudo-code of the one step active sequential learning algorithm  in Algorithm \ref{alg:AdapActSeqLearning}. At each time $t$, we are given an unlabeled data pool $\mathcal{S}_t$ and the estimation of $\hat{\theta}_{t-1}$ in the previous step. We develop the approximate optimal sampling distribution $\hat{\Gamma}_t^*$ using $\hat{\theta}_{t-1}$ as shown by Step 2 in Algorithm \ref{alg:AdapActSeqLearning}. Then, the distribution $\hat{\Gamma}_t^*$ is modified slightly to $\bar{\Gamma}_t$ in Step 3 to ensure it still has the full support of $\mathcal{S}_t$. Using the labeled sample set $\mathcal{S}_t'$, we construct the estimation of $\hat{\theta}_t$ by solving the MLE problem as shown in Step 4.

\section{Theoretical Performance Guarantees}
\label{Sec:Theorem}
In this section, we present the theoretical analysis of Algorithm \ref{alg:AdapActSeqLearning}. We first introduce the assumptions needed. Then, in Section \ref{Sec:active}, we provide the analysis of the active sampling distribution. In Section \ref{Sec:mean_tracking}, we present theoretical guarantees on the sample size selection rules which meet the mean tracking criterion in \eqref{equ:excess_risk_criterion}. In Section \ref{Sec:estimator}, we describe the proposed estimator $\hat{\rho}_t$. The proofs of the theorems and all the supporting lemmas will be presented in the Appendices.
\subsection{Assumptions}
For the purpose of analysis, the following regularity assumption on the log-likelihood function $\ell$ is required to establish the standard Local Asymptotic Normality of the MLE \cite{van2000asymptotic}.
\begin{assum}[Regularity conditions]\label{assump:regularity}
\hfill
\begin{enumerate} \item \textbf{Regularity conditions for MLE:}
\begin{enumerate}
  \item \textbf{Compactness}: $\Theta$ is compact and $\theta_t^*$ is an interior point of $\Theta$ for each $t$.
  \item \textbf{Smoothness}: $\ell(y|x,\theta)$ is smooth in the following sense: the first, second and third derivatives of $\theta$ exist at all interior points of $\Theta$.
  \item \textbf{Strong Convexity}: For each $t$ and $\theta \in \Theta$, $I_{U_t}(\theta)\succeq m I$ with $m>0$, and hence $I_{U_t}(\theta)$ is positive definite and invertible.
  \item \textbf{Boundedness}: For all $\theta\in \Theta$, the largest eigenvalue of $I_{U_t}(\theta)$ is upper bounded by $L_{b}$.
\end{enumerate}
  \item \textbf{Concentration at $\theta_t^*$}: For all $t$, and any $x_t \in \mathcal{S}_t$, $y_t\in \mathcal{Y}$,
  \begin{align}
    \Big\|\nabla\ell(y_t|x_t,\theta_t^*)\Big\|_{I_{U_t}(\theta_t^*)^{-1}}\le L_1\quad \mbox{and}\quad
    \Big\|I_{U_t}(\theta_t^*)^{-1/2} H(x,\theta_t^*) I_{U_t}(\theta_t^*)^{-1/2}\Big\|\le L_2
  \end{align}
  holds with probability one.
  \item \textbf{Lipschitz continuity}: For all $t$, there exists a neighborhood $B_t$ of $\theta_t^*$ and a constant $L_3$, such that for all $x_t \in \mathcal{S}_t$, $H(x_t,\theta)$ are $L_3$-Lipschitz in this neighborhood, namely,
    \begin{align}
      \Big\|I_{U_t}(\theta_t^*)^{-1/2}\big(H(x_t,\theta)-  H(x_t,\theta')\big)I_{U_t}(\theta_t^*)^{-1/2}\Big\|
                    \le L_3\|\theta-\theta'\|_{I_{U_t}(\theta_t^*)}
    \end{align}
    holds for $\theta,\theta' \in B_t$.
\end{enumerate}
\end{assum}

In addition, we need the following assumption to prove that replacing $\theta_t^*$ with $\hat{\theta}_{t-1}$ in Algorithm \ref{alg:AdapActSeqLearning} does not change the performance of the active learning algorithm in terms of the convergence rate. This assumption is satisfied by many classes of models, including the generalized linear model \cite{chaudhuri2015convergence}.

\begin{assum}[Point-wise self-concordance]\label{assump:concordance}
For all $t$, there exists a constant $L_4$, such that
  \begin{align}
    -L_4\| \theta_t-\theta_t^* \|_2 H(x,\theta_t^*) \preceq H(x,\theta_t)-H(x,\theta_t^*)
    \preceq L_4\| \theta_t-\theta_t^* \|_2 H(x,\theta_t^*).
  \end{align}
\end{assum}

\subsection{Optimal Active Learning Sampling Distribution}
\label{Sec:active}
In this subsection, we provide the intuition and analysis of Step 1 in Algorithm \ref{alg:AdapActSeqLearning}. The construction of the active sampling distribution $\Gamma_t$ is motivated by the following lemma, which characterizes the convergence rate of the ERM solution $\hat{\theta}_{\Gamma_t}$ defined in \eqref{equ:MLE} when $\rho$ and $\theta_{t-1}^*$ are known.

%Note that we assume $\theta_{t-1}^*$ and $\rho$ is known and set $\Theta_t= \{ \theta_t| \| \theta_t-\theta_{t-1}^* \| \leq \rho \}$ in Lemma \ref{lem:BoundsWithArbitrarySamplingDistri}.

\begin{lem}\label{lem:BoundsWithArbitrarySamplingDistri}
Suppose Assumptions \ref{cond:IndependCond} and \ref{assump:regularity} hold, and let $\Theta_t \triangleq \{ \theta_t| \| \theta_t-\theta_{t-1}^* \| \leq \rho \}$. For any sampling distribution $\Gamma_t$ on $\mathcal{S}_{t}$, suppose that  $I_{\Gamma_t}(\theta_t^*)\succeq CI_{U_t}(\theta_t^*)$ holds for some constant $C<1$. Then, for sufficiently large $K_t$, such that $\gamma_{t} \triangleq \mathcal{O}\big( \frac{1}{C^2}(L_1L_3+\sqrt{L_2}) \sqrt{\frac{\log dK_t}{K_t}} \big)<1$, the excess risk of $\hat{\theta}_{\Gamma_t}$ can be bounded as
\begin{equation}
(1-\gamma_{t})\frac{\tau_t^2}{K_t} - \frac{L_1^2}{C K_t^2} \leq \mathbb{E}[L_{U_t}(\hat{\theta}_{\Gamma_t}) - L_{U_t}(\theta_t^*)]
\leq (1+\gamma_{t}) \frac{\tau_t^2}{K_t} + \frac{2L_b\rho^2}{K_t^2}
\end{equation}
for all $t$, where $\tau_t^2\triangleq\frac{1}{2}{\mathrm{Tr}}\big( I_{\Gamma_t}^{-1}(\theta_t^*) I_{U_t}(\theta_t^*) \big)$.
\end{lem}

In practice, the parameter space $\Theta_t = \{ \theta_t| \| \theta_t-\theta_{t-1}^* \| \leq \rho \}$ is unknown and the ERM solution of \eqref{equ:MLE} cannot be obtained directly due to the computational issue. To solve these problems, we can apply optimization algorithm such as SGD to find approximate minimizers in the original parameter space $\Theta$ with initialization at $\hat{\theta}_{t-1}$. Thus, we further build Algorithm \ref{alg:AdapActSeqLearning} and our theoretical results with the SGD algorithm (which incidentally achieves the optimal convergence rate for ERM). We need the following assumptions on the optimization algorithm to solve \eqref{equ:MLE}:

\begin{assum}\label{assum:SGD_bound}
Given an optimization algorithm that generates an approximate loss minimizer  $\hat{\theta}_t \triangleq \mathscr{A}\big(\hat{\theta}_{t-1},\ \{ \nabla_\theta \ell(y_{k,t}|x_{k,t},\theta) \}_{k=1}^{K_t}\big)$ using $K_t$ stochastic gradients $\{ \nabla_\theta \ell(y_{i,t}|x_{i,t},\theta) \}_{k=1}^{K_t}$ with initialization at $\hat{\theta}_{t-1}$, if  $\mathbb{E}\|\hat{\theta}_{t-1}-\theta_t^*\|_2^2 \le \Delta_t^2$, there exists a function $b(\tau_t^2, \Delta_t,K_t)$ such that
\begin{equation}
  \mathbb{E}[L_{U_t}(\hat{\theta}_t)]-L_{U_t}(\theta_t^*)
  \le b(\tau_t^2, \Delta_t,K_t),
\end{equation}
where $b(\tau_t^2, \Delta_t,K_t)$ monotonically increases with respect to $\tau_t^2$, $\Delta_t$ and $1/K_t$.
\end{assum}

The bound $b(\tau_t^2, \Delta_t,K_t)$ depends on the converge rate  $\tau_t^2$ and the expectation of the difference between the initialization and the true minimizer $\Delta_t$, which correspond to the first and the second term in the upper bound of Lemma \ref{lem:BoundsWithArbitrarySamplingDistri}, respectively. As an example for this type of bound, for the Streaming Stochastic Variance Reduced Gradient (Streaming SVRG) algorithm in \cite{frostig2015competing}, it holds that 
\begin{equation}
b(\tau_t^2, \Delta_t,K_t)=C_1 \frac{ \tau_t^2}{K_t} +C_2 \big(\frac{\Delta_t}{K_t}\big)^2
\end{equation}
with constant $C_1$ and $C_2$. In addition, the paper \cite{wilson2018adaptive} contains several examples of the bound $b(\tau_t^2, \Delta_t,K_t)$ with other variations of SGD algorithm.

Then, the following theorem characterizes the convergence rate of the active sampling distribution used in Algorithm \ref{alg:AdapActSeqLearning} in the order sense.
%The detailed proof will be presented in the Appendix.

\begin{thm}\label{thm:BoundsWithKnownRho}
Suppose Assumptions \ref{cond:IndependCond}-\ref{assum:SGD_bound} hold, and let $\beta_{t}\triangleq L_4(\rho + \frac{1}{\delta}\sqrt{\frac{2\varepsilon}{m}})<1$. Then,
%1) We have the following lower bound on the excess risk of  $\hat{\theta}_{t}$ in Algorithm \ref{alg:AdapActSeqLearning}
%\begin{equation}\label{equ:LowerBoundOfKnownThm}
%\mathbb{E}[L_{U_t}(\hat{\theta}_t) - L_{U_t}(\theta_t^*)] \geq (1-\gamma_{t})\frac{\mathrm{Tr}\big( I_{\bar{\Gamma}_t}^{-1}(\theta_t^*) I_{U_t}(\theta_t^*) \big) }{2K_t} - \frac{L_1^2}{(1-\alpha_t) K_t}.
%\end{equation}
%\\
the excess risk of $\hat{\theta}_{t}$ in Algorithm \ref{alg:AdapActSeqLearning} is upper-bounded by
\begin{equation}\label{equ:UpperBoundOfKnownThm}
\mathbb{E}[L_{U_t}(\hat{\theta}_t) - L_{U_t}(\theta_t^*)] \leq b(\acute{\tau}_t^{2}, \Delta_t ,K_t),
\end{equation}
with probability 1-$\delta$, where
\begin{equation}
\acute{\tau}_t^{2} = \Big(\frac{1+\beta_{t}}{1-\beta_{t}}\Big)^2 \frac{\mathrm{Tr}\big( I_{\Gamma_t^*}^{-1}(\theta_t^*) I_{U_t}(\theta_t^*) \big)}{2 \alpha_t},\qquad \Delta_t = \sqrt{\frac{2\varepsilon}{m}}+ \rho,
\end{equation}
$\delta \in(0,1)$ and $\Gamma_{t}^*$ is the optimal sampling distribution minimizing ${\mathrm{Tr}}\big( I_{\Gamma_t^*}^{-1}(\theta_t^*) I_{U_t}(\theta_t^*) \big)$.
\end{thm}

\begin{rem}\label{rem:rate}
A comparison between Theorem \ref{thm:BoundsWithKnownRho} and Lemma \ref{lem:BoundsWithArbitrarySamplingDistri} shows that the convergence rate of Algorithm \ref{alg:AdapActSeqLearning} that approximates $\theta_t^*$ with $\hat{\theta}_{t-1}$ in Step 1 is the same as the ERM solution with high probability, as long as the change in the minimizers $\rho$ is small enough, i.e., $L_4(\rho + \frac{1}{\delta}\sqrt{{2\varepsilon}/{m}})<1$. In certain cases such as linear regression model, the Hessian matrices are independent of $\theta_t^*$. Thus, no approximation is needed in constructing the sampling distribution, and Algorithm \ref{alg:AdapActSeqLearning} is rate optimal.
\end{rem}

%Algorithm \ref{alg:AdapActSeqLearning} is rate optimal with high probability.

%We assume that either the constant change in minimizers condition \eqref{equ:EqualChangeRate}, or the bounded change in minimizers condition \eqref{equ:BoundedChangeRate} holds. Our analysis is not affected by which one is true.
%The following theorem demonstrates that for the case where $\rho$ is unknown, the mean tracking criterion holds almost surely.

\subsection{Sample Size Selection Rule}
\label{Sec:mean_tracking}

In this subsection, we explain and analyze the sample selection rule of Step 2 in Algorithm \ref{alg:AdapActSeqLearning}. The idea starts with the bound $b(\tau_t^2, \Delta_t,K_t)$ from Assumption \ref{assum:SGD_bound}. If we can compute $\tau_t^2$ and $\Delta_t$, the sample size $K_t$ can be determined by letting $b(\tau_t^2, \Delta_t,K_t) \le \varepsilon$ to satisfy the mean tracking criterion.
% and proceed inductively using the previous excess risk bound $\varepsilon$ for $t\ge 2$.

However, $\theta_t^*$ in $\tau_t^2=\frac{1}{2}\mathrm{Tr}\big( I_{\Gamma_t}^{-1}(\theta_t^*) I_{U_t}(\theta_t^*) \big)$ is unknown in practice. Although we can approximate $\theta_t^*$ using $\hat{\theta}_{t-1}$ as we did in Step 1, this upper bound only holds with high probability as shown in Theorem \ref{thm:BoundsWithKnownRho}, which means the mean tracking criterion will be satisfied with high probability. To avoid this issue, we use the fact that $\mathrm{Tr}\big( I_{\Gamma_t^*}^{-1}(\theta_t^*) I_{U_t}(\theta_t^*) \big) \le \mathrm{Tr}\big( I_{U_t}^{-1}(\theta_t^*) I_{U_t}(\theta_t^*) \big)= d$ (recall $d$ is the dimension of parameters) to form a conservative bound $b(d/2, \Delta_t,K_t)$ to choose $K_t$, which works for the uniform sampling distribution $U_t$.

%Then, we start to show how to choose $K_t$ adaptively to achieve the target excess risk $\varepsilon$ from a simple case where $\rho$ is known.

To bound the difference between the initialization and the true minimizer $\Delta_t$, we have the inequality $\mathbb{E}\|\hat{\theta}_{t-1}-\theta_t^* \|_2^2 \le (\sqrt{{2\varepsilon}/{m}}+\rho)^2$  following from the triangle inequality, Jensen’s inequality and the strong convexity in Assumption \ref{assump:regularity}. This inequality implies that $\Delta_t = \sqrt{{2\varepsilon}/{m}}+\rho$.

Therefore, if $\rho$ is known, we can set $K_t^* = \min\Big\{K\ge 1 \Big| b\Big(d/2, \sqrt{\frac{2\varepsilon}{m}}+\rho,K \Big) \le \varepsilon \Big\}$ for $t\ge 2$ to ensure that $\mathbb{E}[L_{U_t}(\hat{\theta}_t) - L_{U_t}({\theta}_t^*)]\le \varepsilon$. For $t=1$, we could always use $\mathrm{diameter}(\Theta)$ to bound $\Delta_1$ and select $K_1$. In general, if $\rho$ is much smaller than $\mathrm{diameter}(\Theta)$, then we require significantly fewer samples $K_t$ to meet the mean tracking criterion for $t\ge 2$.

For the case where the change of the minimizers $\rho$ is unknown, we could replace $\rho$ with an estimate $\hat{\rho}_{t-1}$ to select the sample size. The following theorem characterizes the convergence guarantee using the sample size selection rule of step 2 in Algorithm \ref{alg:AdapActSeqLearning} and the estimator of $\hat{\rho}_t$ in Section \ref{Sec:estimator}.

\begin{thm}\label{thm:mean_tracking}
If
\begin{equation}
K_t \ge K_t^* \triangleq \min\Big\{K\ge 1 \Big| b\Big(d/2, \sqrt{\frac{2\varepsilon}{m}}+\hat{\rho}_{t-1}, K \Big) \le \varepsilon \Big\},
\end{equation}
 then for all $t$ large enough we have $\limsup_{t\to \infty} \big(\mathbb{E}[L_{U_t}(\hat{\theta}_t)] - L_{U_t}({\theta}_t^*)\big)\le \varepsilon$ almost surely.
\end{thm}

%In the following subsections, we provide some intuitions to understand the sample selection rule in Theorem \ref{thm:mean_tracking} when $\rho$ is known, and describe the proposed estimator of $\hat{\rho}_t$ in detail.

%\subsubsection{Change in Minimizers Known}

\subsection{Estimating the Change in Minimizers}
\label{Sec:estimator}
In this subsection, we construct an estimate $\hat{\rho}_t$ of the change in the minimizers $\rho$ using the active learning samples for step 5 in Algorithm \ref{alg:AdapActSeqLearning}.

We first construct an estimate $\widetilde{\rho}_t$ for the one-step changes $\|\theta_{t-1}^*-\theta_{t}^*\|$. As a consequence of strong convexity, the following lemma holds.

%First, we construct estimates $\widetilde{\rho}_t$ for the one step changes $\|\theta_{t-1}^*-\theta_{t}^*\|$ for $t\ge 2$. Then, we combine the one step estimates to construct an overall estimate $\hat{\rho}_t$ for $\rho$.

%With this property, we then extend our results in Section \ref{sec:SeqActLearning_known} to the case where $\rho$ is unknown and we show that similar mean tracking criterion can be achieved.
%We show that for all $t$ large enough with appropriately chosen sequences $\{r_t\}$, our estimates satisfies $\hat{\rho}_t +r_t \ge \rho$ almost surely.

\begin{lem}\label{lem:rho_upper_bound}
Suppose  Assumption \ref{assump:regularity}  holds, then
\begin{equation}
  \|\theta_{t-1}^*-\theta_{t}^*\|^2  \le \frac{1}{m}\big[L_{U_t}(\theta^*_{t-1})- L_{U_t}(\theta^*_{t})
 +  L_{U_{t-1}}(\theta^*_{t})-L_{U_{t-1}}(\theta^*_{t-1})\big].
\end{equation}
\end{lem}
Motivated by Lemma \ref{lem:rho_upper_bound}, we can construct the following one-step estimation of $\rho^2$
\begin{equation}\label{equ:rho_estimator}
  \widetilde{\rho}_t^2 = \frac{1}{m}\big[\hat{L}_{U_t}(\hat{\theta}_{t-1}) - \hat{L}_{U_t}(\hat{\theta}_{t})+\hat{L}_{U_{t-1}}(\hat{\theta}_{t})-\hat{L}_{U_{t-1}}(\hat{\theta}_{t-1})\big],
\end{equation}
where we use 
\begin{equation}
\hat{L}_{U_t}(\hat{\theta}_{t-1}) \triangleq \frac{1}{K_t}\sum_{k=1}^{K_t} \frac{\ell (Y_{k,t}|X_{k,t},\hat{\theta}_{t-1})}{N_t\bar{\Gamma}_t(X_{k,t})}
\end{equation}
as the empirical estimation of $L_{U_t}(\theta_{t-1}^*)$. Note that we are using the samples generated from the active learning distribution, i.e., $X_{k,t} \sim \bar{\Gamma}_t$ and $Y_{k,t} \sim p(Y|X_{k,t},\theta_t^*)$. Thus,  based on the idea of importance sampling \cite{cappe2008adaptive}, we need to normalize the estimate with the sampling distribution $\bar{\Gamma}_t$.

%\subsubsection{Combining One Step Estimates}
%As a special case, we look at combining the one step estimates when \eqref{equ:EqualChangeRate} holds. Under \eqref{equ:EqualChangeRate}, we construct an estimate by averaging the one step estimates
%\begin{equation}\label{equ:Constantbounded_estimator}
%  \hat{\rho}_t^2 \triangleq \frac{1}{t-1} \sum_{i=2}^t \widetilde{\rho}_i^2.
%\end{equation}
%We want to show that for an appropriate sequence $\{r_t\}$, described in Theorems \ref{thm:rho_as} below, for all $t$ large enough $\hat{\rho}_t^2 +r_t \ge \rho$ almost surely under \eqref{equ:EqualChangeRate}. The difficulty in actually proving this convergence result for the proposed estimate in \eqref{equ:rho_estimator} is that bounding the difference between $\mathbb{E}\{\hat{L}_{U_t}(\hat{\theta}_{t-1})\}$ and ${L}_{U_t}(\theta_{t-1}^*)$.
%
%\subsubsection{Combining with Bounded Change of Minimizers}
%We consider estimating $\rho$ in the case that \eqref{equ:BoundedChangeRate} holds.

%In this case, we do produce an upper bound, but it increases to the trivial bound $\mathrm{Diameter}(\Theta)^2$.
Then, we combine the one-step estimates to construct an overall estimate. The simplest way to combine the one-step estimates would be to set $\acute{\rho}_t^2 = \max\{\widetilde{\rho}_2^2,\cdots, \widetilde{\rho}_t^2 \}$. However, if we suppose that each estimate $\widetilde{\rho}$ is an independent Gaussian random variable, then this estimate goes to infinity as $t \to \infty$. To avoid this issue, we use a class of functions $h_W : \mathbb{R}^W\to \mathbb{R}$ that are non-decreasing in their arguments
and satisfy $\mathbb{E}[h_W(\rho_j, \cdots , \rho_{j-W+1})] \ge \rho$. For example, $h_W(\rho_j, \cdots , \rho_{j-W+1})=\frac{W+1}{W}\max\{\rho_j, \cdots , \rho_{j-W+1}\}$ satisfies the requirements. The combined estimate of $\acute{\rho}_t^2$ is computed
by applying the function $h_W$ to a sliding window of one-step estimates of $\widetilde{\rho}^2$, i.e.,
\begin{equation}\label{equ:bounded_estimator}
  \acute{\rho}_t^2 = \frac{1}{t-1}\sum_{j=2}^t h_{\{\min[W,j-1]\}}(\widetilde{\rho}_j^2,\widetilde{\rho}_{j-1}^2,\cdots, \widetilde{\rho}_{\max[j-W+1,2]}^2).
\end{equation}
%In the appendix, we demonstrate that similar almost surely convergence results can be established as in Theorem \ref{thm:rho_as} for the case where \eqref{equ:BoundedChangeRate} holds.
The following theorem characterizes the performance of proposed estimator in \eqref{equ:bounded_estimator}.
\begin{thm}\label{thm:rho_as}
Suppose Assumptions \ref{cond:IndependCond} and \ref{assump:regularity} hold, and there exists a sequence $\{r_t\}$ \footnote{Note that a choice of $r_t$ that is greater than $1/\sqrt{t-1}$ in the order sense works here.} satisfying
\begin{equation*}
\sum_{t=1}^\infty  \exp \Big\{-\frac{2m^2(t-1)r_t^2}{9L_b^2 \mathrm{Diameter}^4(\Theta)} \Big\} < \infty
\end{equation*}
for all $t$ large enough, then $\hat{\rho}_t^2 \triangleq \acute{\rho}_t^2 + D_t+ r_t \ge \rho^2$ almost surely with a constant $D_t$.
\end{thm}

\input{Experiments}
\vspace{-0.15cm}
\section{Conclusions}
\label{Sec:conclusion}
\vspace{-0.15cm}
In this paper, we propose an active and adaptive learning framework to solve a sequence of learning problems, which ensures a bounded excess risk for each individual learning task when the number of time steps is sufficiently large. We construct an estimator of the change in the minimizers $\hat{\rho}_t$ using active learning samples  and show that this estimate upper bounds the true parameter $\rho$ almost surely. We test our algorithm on a synthetic regression problem, and further apply it to a recommendation system that tracks changes in preferences of customers. Our experiments demonstrate that our algorithm achieves better performance compared to the other baseline algorithms.

%\subsubsection*{Acknowledgments}

\newpage
\small
\bibliography{active_learning}
\bibliographystyle{plainnat}

\normalsize
\input{appendix}

\end{document}

%% file: Experiments.tex
\section{Experiments}
\label{Sec:experiments}
In this section, we present two experiments to validate our algorithm and the related theoretical results: one is to track a synthetic regression model and the other is to track the time-varying user preferences in a recommendation system. {More experiments on binary classification are presented in the Appendices.} We use three baseline algorithms for comparison: passive adaptive algorithm, active random algorithm and passive random algorithm. Compared with Algorithm \ref{alg:AdapActSeqLearning}, \emph{Passive} means drawing new samples using a uniform distribution $U_t$ in Step 3 and \emph{Random} means replacing the estimate of $\hat{\theta}_{t-1}$ with a random point from $\Theta$ in Step 1 and 4. All reported results are averaged over 1000 runs of Monte Carlo trials. The sizes of the sample pools for all the test algorithms are the same with $N_t=500$, and the number of considered time steps is 25. We construct the active sampling distribution with the exact solution of the SDP problem in Step 1. Note that approximation algorithms for SDP introduced in \cite{sourati2017asymptotic} can be applied to accelerate this process. We set $K_t = K_t^*$ for all the test algorithms and use the estimator defined in Section \ref{Sec:estimator} with window size $W=3$ to estimate $\rho$.

%In experiments, the active sampling distribution is given by solving the SDP shown in Algorithm \ref{alg:AdapActSeqLearning} using $N_t$ samples.

%As shown in \ref{thm:mean_tracking}, we set $K_t = K^*$ to ensure that the mean tracking criterion is met.
%In experiments, the active sampling distribution is given by solving the SDP shown in Algorithm \ref{alg:AdapActSeqLearning} using $N_t$ samples.
%\subsection{Data sets}

\subsection{Synthetic Regression}
The model of the synthetic regression problem is $y_t = \theta_t^{T} x_t + w_t$, where the input variable $x_t\thicksim \mathcal{N}(0,0.1I)$ is a 5-dimensional Gaussian vector  and the noise $w_t\thicksim \mathcal{N}(0,0.5)$. We consider learning the parameter $\theta_t$ by minimizing the following negative log-likelihood function $\ell(y_{k,t}|x_{k,t},\theta_t) = (y_{k,t}-\theta_t^{T}x_{k,t})^2$. In the simulations, the change of the true minimizers is $\rho=10$, and the target excess risk is $\varepsilon=1$. To highlight the time-varying nature of the problem, we implement the ``all samples up front'' method by using $\sum_{t=1}^{25} K_t^*$ samples at the first time step and keep this time-invariant regression model for the rest of considered time steps.

Fig. \ref{fig:RegPerformance} shows that using $K_t^*$ new samples, the passive adaptive algorithm meets the mean tracking criterion and our proposed active and adaptive learning algorithm outperforms all the other algorithms. The ``all samples up front'' algorithm outperforms the other algorithms initially, but it fails to track the time-varying underlying model after only a few time steps. Moreover, the excess risk of active random algorithm is almost the same as that of active adaptive algorithm, since the Hessian matrices in the regression task are independent of $\theta_t$. In this case, no approximation is needed and the change rate $\rho$ in the regression task can be arbitrarily large, as we mentioned in Remark \ref{rem:rate}. Fig \ref{fig:RegRhoEst} shows that $\hat{\rho}_t$ converges to a conservative estimate of $\rho$, which verifies Theorem \ref{thm:rho_as}. Moreover, the corresponding number of samples determined by Theorem \ref{thm:mean_tracking} is depicted in Fig. \ref{fig:RegNumSample}, which shrinks adaptively as $\hat{\rho}_t$ converges.

\begin{figure}[htp]
\centering     %%% not \center
\subfigure[]{\label{fig:RegPerformance}\includegraphics[width=0.32\textwidth]{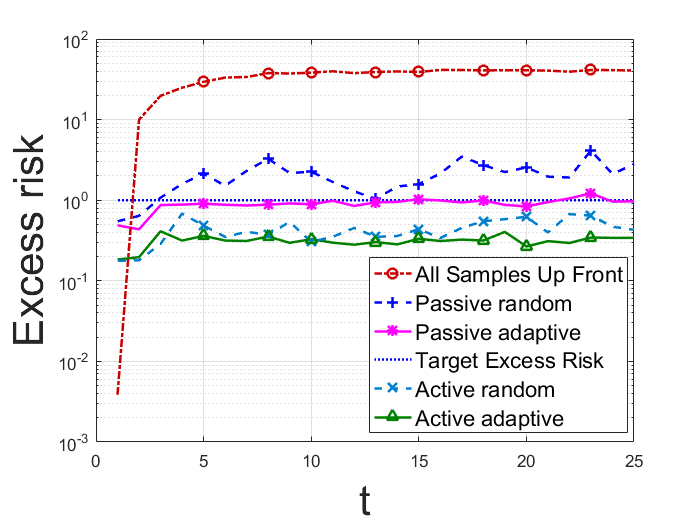}}
\subfigure[]{\label{fig:RegRhoEst}\includegraphics[width=0.32\textwidth]{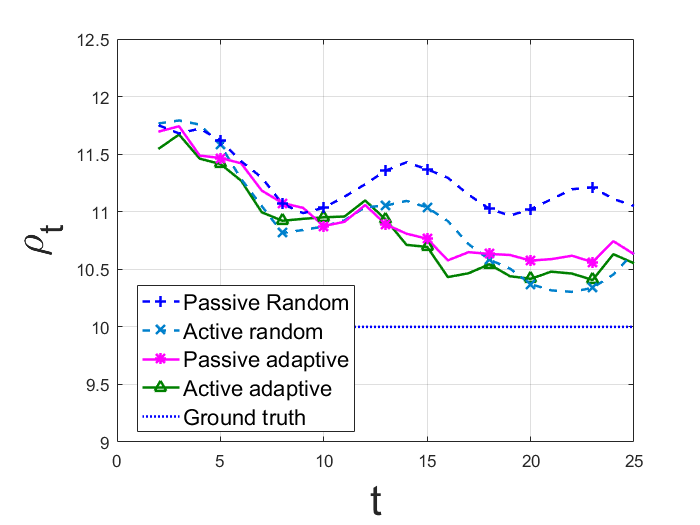}}
\subfigure[]{\label{fig:RegNumSample}\includegraphics[width=0.32\textwidth]{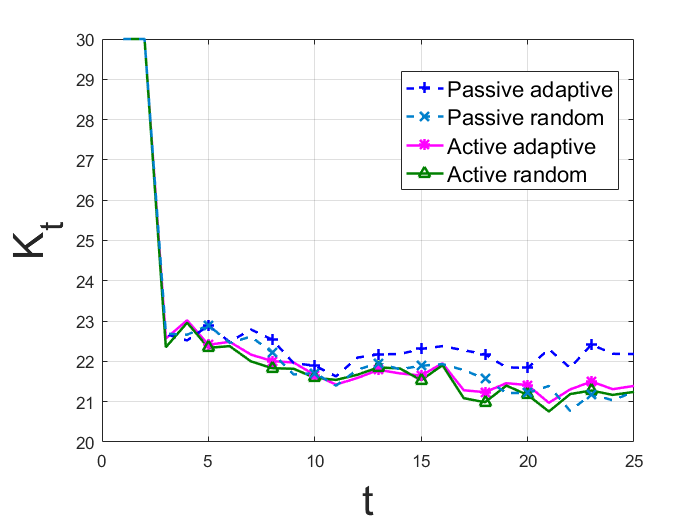}}
\subfigure[]{\label{fig:RealClassification}\includegraphics[width=0.32\textwidth]{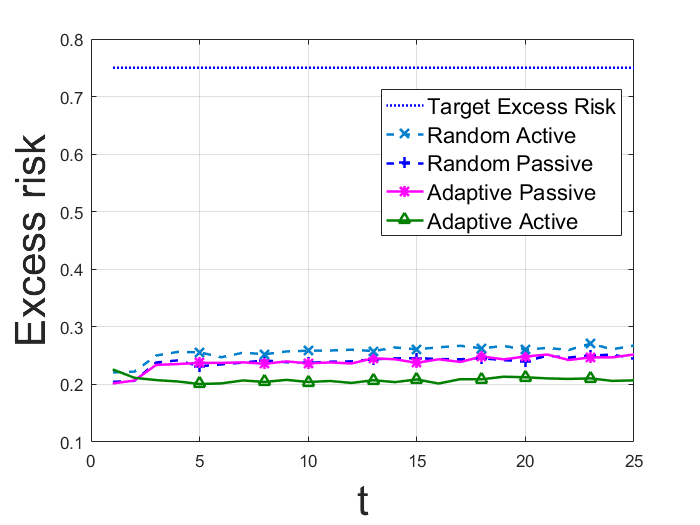}}
\subfigure[]{\label{fig:RealClasRhoEst}\includegraphics[width=0.32\textwidth]{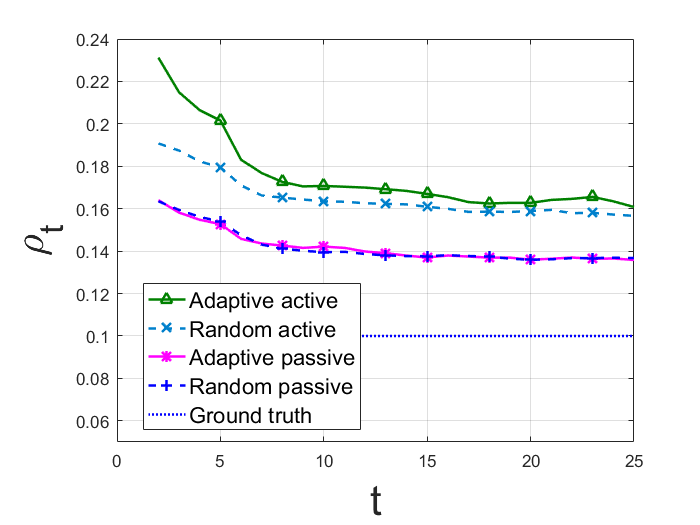}}
\subfigure[]{\label{fig:RealClassificationError}\includegraphics[width=0.32\textwidth]{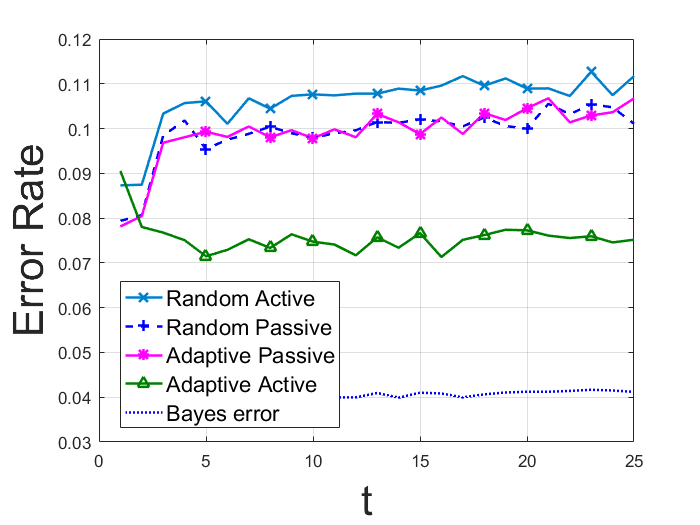}}
\caption{Experiments on synthetic regression: (a) Excess risk. (b) Estimated rate of change of minimizers. (c) Number of samples. Experiments on user preference tracking performance using Yelp data: (d) Excess risk. (e) Estimated rate of change of minimizers. (f) Classification error.}\label{fig:RegFigures}
\vspace{-0.35cm}
\end{figure}

\subsection{Tracking User Preferences in Recommendation System}

We utilize a subset of Yelp 2017 dataset {\footnote{https://www.yelp.com/dataset}} to  perform our experiments. We censor the original dataset such that each user has at least 10 ratings. After censoring procedure, our dataset contains ratings of $M = 473$ users for $N = 858$ businesses. By converting the original 5-scale ratings to a binary label for all businesses with high ratings (4 and 5) as positive ($1$) and low ratings (3 and below) as negative ($-1$), we form the $N \times M$ binary rating matrix $R$, which is very sparse and only $2.6\%$ are observed. We complete the sparse matrix $R$ to make recommendations by using the matrix factorization method \cite{koren2009matrix}. The rating matrix $R$ can be modeled by the following logistic regression model
\begin{equation}\label{equ:MF}
p(R_{u,b}|\phi_b,\phi_u ) = \frac{1}{1 + \exp^{-R_{u,b}\phi_u^\top\phi_b}},
\end{equation}
where $\phi_u$ and $\phi_b$ are the $d$-dimensional latent vectors representing the preferences of user $u$ and properties of business $b$, respectively. Then, we train $\phi_u$ and $\phi_b$ with dimension $d=5$ for each user and business in the dataset using maximum likelihood estimation by SGD. With the learned latent vectors, we can complete the matrix $R$ and make recommendations to customers in a collaborative filtering fashion \cite{elahi2016survey,rubens2015active}.

In practice, the preferences of users $\phi_{u,t}$ may vary with time $t$, and hence user features need to be retrained. Considering the fact that acquiring new ratings of users can be expensive, we apply our active and adaptive learning algorithm to further reduce the number of new samples while maintaining the mean tracking accuracy.

In the following experiment, we use a random subset of $\{\phi_b\}$ with size $N_t$ as our unlabeled data pool, while the remaining serve as a test set to evaluate the algorithms. To model the bounded time-varying changes of user preferences $\phi_{u,t}$, we start from a randomly chosen user feature and update it by adding a random Normal drift with norm bounded by 0.1 at each time step. Since we are unable to retrieve the actual answer from a real user, we generate the labels with the probabilistic model given by \eqref{equ:MF} with true parameter $\phi_{u,t}$ instead. Note that one cannot ask a user the same question twice in a real recommendation system, and therefore we implement without replacement sampling by querying the labels of the samples having the largest $K_t^*$ values in the active sampling distribution $\bar{\Gamma}_t$.

%Similar as the results in synthetic regression,
Fig. \ref{fig:RealClasRhoEst} shows that $\hat{\rho}_t$ converges to a conservative estimate of $\rho$, and the corresponding sample size converges to $K_t^*=13$ after two time steps. Fig. \ref{fig:RealClassification} and Fig. \ref{fig:RealClassificationError} show that our algorithm achieves a error rate of 6\% with these samples and significantly outperforms the other algorithms. This is because the Hessian matrices of logistic regression are functions of $\theta_t$, and hence the sampling distribution generated by the active and adaptive algorithm selects more informative samples.

%Note that in real data, each coordinate in features may be dependent with each other, then the independent dimension is less than the dimension of features. In this case, the upper bound for selecting the number of samples may be loose, and hence the empirical excess risk can be much smaller than the target excess risk constraint.

%% file: appendix.tex
\newpage
\appendix
\section{Proof of Lemma \ref{lem:BoundsWithArbitrarySamplingDistri}}
\label{appx:lemma_1}

To prove Lemma \ref{lem:BoundsWithArbitrarySamplingDistri}, we use the following result from \cite{frostig2015competing}. In particular, the following lemma is a generalization of Theorem 5.1 in \cite{frostig2015competing}, and its proof follows from generalizing the derivation of that theorem and is omitted here.

\begin{lem}\label{lem:OriginalConvergence}
Suppose $\psi_1(\theta),\cdots,\psi_K(\theta):\mathbb{R}^d\to \mathbb{R}$ are random functions drawn i.i.d. from a distribution, where $\theta \in \Theta \subseteq \mathbb{R}^d$. Denote $P(\theta) = \mathbb{E}[\psi(\theta)]$ and let $Q(\theta):\mathbb{R}^d\to \mathbb{R}$ be another function. Let
\begin{equation*}
  \hat{\theta} = \argmin_{\theta\in \Theta} \sum_{k=1}^K \psi_k(\theta), \quad \mbox{and}\ \theta^*=\argmin_{\theta\in \Theta}P(\theta).
\end{equation*}
Assume:
\begin{enumerate}
  \item \textbf{Regularity conditions}:
  \begin{enumerate}
    \item Compactness: $\Theta$ is compact, and $\theta^*$ is an interior point of $\Theta$.
    \item Smoothness: $\psi(\theta)$ is smooth in the following sense: the first, second and third
derivatives exist at all interior points of $\Theta$ with probability one.
    \item Convexity: $\psi(\theta)$ is convex with probability one, and $\nabla^2P(\theta^*)$ is positive definite.
    \item $\nabla P(\theta^*)=0$ and $\nabla Q(\theta^*)=0$.
  \end{enumerate}
  \item \textbf{Concentration at $\theta^*$}: Suppose
  \begin{equation*}
    \Big\|\nabla\psi(\theta^*)\Big\|_{\nabla^2P(\theta^*)^{-1}}\le L_1'\quad \mbox{and}\quad \Big\|\big(\nabla^2P(\theta^*)\big)^{-1/2}\nabla^2 \psi(\theta^*)\big(\nabla^2P(\theta^*)\big)^{-1/2}\Big\|_2\le L_2'
  \end{equation*}
  hold with probability one.
  \item \textbf{Lipschitz continuity}: There exists a neighborhood $B$ of $\theta^*$ and a constant $L_3'$, such that $\nabla^2 \psi(\theta)$ and $\nabla^2 Q(\theta)$ are $L_3'$-Lipschitz in this neighborhood, namely,
    \begin{equation*}
    \begin{split}
      &\Big\|\big(\nabla^2P(\theta^*)\big)^{-1/2}\big(\nabla^2 \psi(\theta)- \nabla^2 \psi(\theta')\big)\big(\nabla^2P(\theta^*)\big)^{-1/2}\Big\|_2 \le L_3'\|\theta-\theta'\|_{\nabla^2P(\theta^*)},\\
      &\Big\|\big(\nabla^2Q(\theta^*)\big)^{-1/2}\big(\nabla^2 Q(\theta)- \nabla^2 Q(\theta')\big)\big(\nabla^2Q(\theta^*)\big)^{-1/2}\Big\|_2 \le L_3'\|\theta-\theta'\|_{\nabla^2P(\theta^*)},
    \end{split}
    \end{equation*}
    holds with probability one, for $\theta,\theta' \in B$,
\end{enumerate}
Choose $p \ge 2$ and define
\begin{equation*}
  \gamma \triangleq  c(L_1' L_3'+\sqrt{L_2'}) \sqrt{{\frac{p\log dK}{K}}},
\end{equation*}
where $c$ is an appropriately chosen constant. Let $c'$ be another appropriately chosen constant. If $K$
is large enough so that $\sqrt{\frac{p\log dK}{K}}\le c' \min\left\{\frac{1}{\sqrt{L_2'}},\frac{1}{L_1'L_3'},\frac{\mathrm{diameter}(B)}{L_1'} \right\}$, then:
\begin{equation*}
  (1-\gamma)\frac{\tau^2}{K}-\frac{L_1'^2}{K^{p/2}} \leq \mathbb{E}\big[Q(\hat{\theta})-Q(\theta^*)\big]  \leq (1+\gamma)\frac{\tau^2}{K}+\frac{\max_{\theta\in \Theta}\left[Q(\theta)-Q(\theta^*)\right]}{K^p},
\end{equation*}
where
\begin{equation*}
  \tau^2 \triangleq \frac{1}{2K^2} {\mathrm{Tr}}\bigg(\sum_{i,j}\mathbb{E}\big[\nabla\psi_i(\theta^*)\nabla\psi_j(\theta^*)^\top\big] \big(\nabla^2P(\theta^*)\big)^{-1} \nabla^2Q(\theta^*) \big(\nabla^2P(\theta^*)\big)^{-1}\bigg).
\end{equation*}
\end{lem}

Then, we proceed to prove Lemma \ref{lem:BoundsWithArbitrarySamplingDistri}.

\begin{proof}[Proof of Lemma \ref{lem:BoundsWithArbitrarySamplingDistri}]
We first use Lemma \ref{lem:OriginalConvergence} to bound the excess risk, which is similar to the idea of Lemma 1 in \cite{chaudhuri2015convergence}. We first define
\begin{equation}
\psi_k(\theta_{t})=\ell (Y_{k,t}|X_{k,t},\theta_{t}),
\end{equation}
where $X_{k,t}\sim \Gamma_t$ and $Y_{k,t} \sim p(Y_{k,t}|X_{k,t},\theta_t^*)$ for $1\le k \le K_t$. Then,
\begin{equation}
P(\theta_t)=\mathbb{E}(\psi_k(\theta_{t}))=L_{\Gamma_t}(\theta_t),\quad \mbox{and} \quad \nabla^2P(\theta_t^*) = I_{\Gamma_t}(\theta_t^*).
\end{equation}
Further, we choose
\begin{equation}
Q(\theta_t)=L_{U_t}(\theta_t), \quad \mbox{and} \quad \nabla^2Q(\theta_t^*) = I_{U_t}(\theta_t^*).
\end{equation}
As shown in Assumption \ref{assump:regularity}, the assumptions of Lemma \ref{lem:OriginalConvergence} are satisfied. Moreover, according to the condition that $I_{\Gamma_t}(\theta^*)\succeq CI_{U_t}(\theta^*)$ holds for some constant $C<1$ in Lemma \ref{lem:BoundsWithArbitrarySamplingDistri}, we have
\begin{align}
\Big\|I_{\Gamma_t}(\theta_t^*)^{-1/2}&\big(H(x,\theta_t)-H(x,\theta_t')\big)I_{\Gamma_t}(\theta_t^*)^{-1/2}\Big\|_2 \nn\\
&\le \frac{1}{C} \Big\|I_{U_t}(\theta_t^*)^{-1/2}\big(H(x,\theta_t)- H(x,\theta_t')\big)I_{U_t}(\theta_t^*)^{-1/2}\Big\|_2 \nn \\
&\le \frac{L_3}{C}\|\theta-\theta'\|_{I_{U_t}(\theta_t^*)} \le \frac{L_3}{C^{3/2}}\|\theta-\theta'\|_{I_{\Gamma_t}(\theta_t^*)}
\end{align}
and
\begin{align}
\Big\|I_{U_t}(\theta_t^*)^{-1/2}&\big(H(x,\theta_t)- H(x,\theta_t')\big)I_{U_t}(\theta_t^*)^{-1/2}\Big\|_2  \nn \\
&\leq L_3\|\theta-\theta'\|_{I_{U_t}(\theta_t^*)}\leq \frac{L_3}{\sqrt{C}}\|\theta-\theta'\|_{I_{\Gamma_t}(\theta_t^*)}.
\end{align}
Hence, $L_3'=\max\{{L_3}/{C^{3/2}},{L_3}/{\sqrt{C}}\}={L_3}/{C^{3/2}}$. Similarly, we have $L_1'={L_1}/{\sqrt{C}}$ and $L_2'={L_2}/{C}$. In summary, the Assumptions 2 and 3 in Lemma \ref{lem:OriginalConvergence} are satisfied with constants
\begin{equation}
(L_1',L_2',L_3')=(L_1/\sqrt{C},L_2/C,L_3/C^{3/2}).
\end{equation}

Applying Lemma \ref{lem:OriginalConvergence} with $p=2$ and considering the fact that $\mathbb{E}_{x\sim\Gamma_t}\big[\nabla \ell(Y_{i,t}|X_{i,t},\theta_t^*)\nabla \ell (Y_{i,t}|X_{i,t},\theta_t^*)^\top\big] = I_{\Gamma_t}(\theta_t^*)$,
\begin{equation}\label{equ:ERM_original}
  (1-\gamma_t)\frac{\tau_t^2}{K_t}-\frac{L_1^2}{CK_t^{2}} \leq \mathbb{E}\big[L_{U_t}(\hat{\theta}_{\Gamma_t})-L_{U_t}(\theta_t^*)\big]  \leq (1+\gamma_t)\frac{\tau_t^2}{K_t}+\frac{\max_{\theta\in \Theta_t} \left[L_{U_t}(\theta)-L_{U_t}(\theta_t^*)\right]}{K_t^2}
\end{equation}
holds, where
\begin{equation}
\gamma_t =\mathcal{O}\bigg((L_1'L_3'+\sqrt{L_2'}) \sqrt{\frac{\log dK_t}{K_t}} \bigg) =\mathcal{O}\bigg( \frac{1}{C^2}(L_1L_3+\sqrt{L_2}) \sqrt{\frac{\log dK_t}{K_t}} \bigg),
\end{equation}
and $\tau_t^2 = \frac{1}{2}{\mathrm{Tr}}\Big(\big(I_{\Gamma_t}(\theta_t^*)\big)^{-1} I_{U_t}(\theta_t^*) \Big)$.

Note that if we assume the parameter set $\Theta_t \triangleq \{ \theta_t| \| \theta_t-\theta_{t-1}^* \| \leq \rho \}$ is known, then the second term in the right hand side of \eqref{equ:ERM_original} can be further bounded as
\begin{equation}
\frac{\max_{\theta\in \Theta_t} \left[L_{U_t}(\theta)-L_{U_t}(\theta_t^*)\right]}{K_t^2} \le \frac{\max_{\theta\in \Theta_t} \left[L_b\|\theta-\theta_t^*\|^2 \right]}{2K_t^2}
\le  \frac{L_b\mathrm{Diameter}(\Theta_t)^2 }{2K_t^2}\le  \frac{2 L_b \rho^2 }{K_t^2},
\end{equation}
where the inequalities follow from the boundedness condition in Assumption \ref{assump:regularity}. Combining this result with the inequality in \eqref{equ:ERM_original} completes the proof of Lemma \ref{lem:BoundsWithArbitrarySamplingDistri}.
\end{proof}

\section{Proof of Theorem \ref{thm:BoundsWithKnownRho}}
\label{appx:thm_1}
\begin{proof}[Proof of Theorem \ref{thm:BoundsWithKnownRho}]
The proof starts from the bound $b(\tau^2, \Delta_t,K_t)$ of the SGD algorithm in Assumption \ref{assum:SGD_bound}. To compute the convergence rate $\tau^2$, we need to first study the approximation of $\theta_t^*$ using $\hat{\theta}_{t-1}$. The difference between $\hat{\theta}_{t-1}$ and $\theta_{t}^*$ can be bounded as
\begin{equation}
\big\| \hat{\theta}_{t-1}-\theta_{t}^* \big\|_2 \leq \big\| \theta_{t-1}^*-\theta_{t}^* \big\|_2 + \big\| \hat{\theta}_{t-1}-\theta_{t-1}^* \big\|_2 \le \rho + \big\| \hat{\theta}_{t-1}-\theta_{t-1}^* \big\|_2.
\end{equation}
To bound the second term, we use the strongly convexity assumption in Assumption \ref{assump:regularity},
%the Taylor expansion of $\ell(y|x,\hat{\theta}_{t-1})$ at $\theta_{t-1}^*$ and
\begin{equation}
\big\| \hat{\theta}_{t-1}-\theta_{t-1}^* \big\|_2^2 \leq \frac{2}{m}(L_{U_{t-1}}\big(\hat{\theta}_{t-1})-L_{U_{t-1}}({\theta}_{t-1}^*)\big).
\end{equation}
Suppose the excess risk bound $\mathbb{E}[L_{U_{t-1}}(\hat{\theta}_{t-1}) - L_{U_{t-1}}({\theta}_{t-1}^*)]\le \varepsilon$ holds for $t-1$. Then, we have
\begin{equation}\label{equ:Delta_t}
\mathbb{E}(\big\| \hat{\theta}_{t-1}-\theta_{t-1}^* \big\|_2) \le \sqrt{\mathbb{E} (\big\| \hat{\theta}_{t-1}-\theta_{t-1}^* \big\|_2^2)} \le \sqrt{{2\varepsilon}/{m}}.
\end{equation}
Then, $\big\| \hat{\theta}_{t-1}-\theta_{t-1}^* \big\|_2 \leq \frac{1}{\delta}\sqrt{\frac{2\varepsilon}{m}}$ holds with probability $1-\delta$ by Markov's inequality, for $\forall\delta\in(0,1)$. Thus,
\begin{equation}\label{equ:ExpectationDiffOfParameters}
  \big\| \hat{\theta}_{t-1}-\theta_{t}^* \big\|_2 \leq \rho+\frac{1}{\delta}\sqrt{\frac{2\varepsilon}{m}}
\end{equation}
holds with probability $1-\delta$.
By the self-concordance condition in Assumption \ref{assump:concordance}, we have that
\begin{equation}
(1-\beta_t)H(x_t,\theta_{t}^*) \preceq H(x_t,\hat{\theta}_{t-1}) \preceq (1+\beta_t)H(x_t,\theta_{t}^*),\quad x_t\in \mathcal{S}_t,
\end{equation}
holds with probability $1-\delta$, where $\beta_t = L_4(\rho+\frac{1}{\delta}\sqrt{\frac{2\varepsilon}{m}})$. Then, for distribution $\Gamma_t^*$, $\hat{\Gamma}_t^*$ and $U_t$, we have
\begin{equation}\label{equ:SelfConcerdence1}
(1-\beta_t)I_{\Gamma_t^*}(\theta_{t}^*) \preceq I_{\Gamma_t^*}(\hat{\theta}_{t-1}) \preceq (1+\beta_t)I_{\Gamma_t^*}(\theta_{t}^*),
\end{equation}
\begin{equation}\label{equ:SelfConcerdence2}
(1-\beta_t)I_{\hat{\Gamma}_t^*}(\theta_{t}^*) \preceq I_{\hat{\Gamma}_t^*}(\hat{\theta}_{t-1}) \preceq (1+\beta_t)I_{\hat{\Gamma}_t^*}(\theta_{t}^*),
\end{equation}
\begin{equation}\label{equ:SelfConcerdence3}
(1-\beta_t)I_{U_t}(\theta_{t}^*) \preceq I_{U_t}(\hat{\theta}_{t-1}) \preceq (1+\beta_t)I_{U_t}(\theta_{t}^*).
\end{equation}
Recall that $\bar{\Gamma}_t=\alpha_{t}\hat{\Gamma}_t^*+(1-\alpha_{t})U_t$. Hence, $I_{\bar{\Gamma}_t}(\theta_t^*) \succeq \alpha_{t} I_{\hat{\Gamma}_t^*}(\theta_t^*)$ which implies that $I_{\bar{\Gamma}_t}(\theta_t^*)^{-1} \preceq \frac{1}{\alpha_{t}} I_{\hat{\Gamma}_t}(\theta_t^*)^{-1}$. Thus,
\begin{equation}\label{equ:UBThmStep1}
\tau_t^2 = \frac{1}{2}\mathrm{Tr}\big(I_{\bar{\Gamma}_t}^{-1}(\theta_t^*) I_{U_t}(\theta_t^*) \big) \le \frac{1}{2\alpha_{t}}\mathrm{Tr}\big(I_{\hat{\Gamma}_t^*}^{-1}(\theta_t^*) I_{U_t}(\theta_t^*) \big).
\end{equation}
From \eqref{equ:SelfConcerdence2} and \eqref{equ:SelfConcerdence3}, \eqref{equ:UBThmStep1} can be further upper bounded by
\begin{align}\label{equ:UBThmStep2}
\mathrm{Tr}\big(I_{\hat{\Gamma}_t^*}^{-1}(\theta_t^*) I_{U_t}(\theta_t^*) \big)
&\le  \frac{1+\beta_t}{1-\beta_t} {\mathrm{Tr}}\big(I_{\hat{\Gamma}_t^*}^{-1} (\hat{\theta}_{t-1}) I_{U_t}(\hat{\theta}_{t-1}) \big)  \nn \\
& \overset{(a)}{\le}  \frac{1+\beta_t}{1-\beta_t} {\mathrm{Tr}}\big(I_{\Gamma_t^*}^{-1} (\hat{\theta}_{t-1}) I_{U_t}(\hat{\theta}_{t-1}) \big) \nn \\
& \overset{(b)}{\le} \left(\frac{1+\beta_t}{1-\beta_t}\right)^2 {\mathrm{Tr}}\big(I_{\Gamma_t^*}^{-1} (\theta_t^*) I_{U_t}(\theta_t^*) \big),
\end{align}
where (a) is because that $\hat{\Gamma}_t^*$ is the minimizer of ${\mathrm{Tr}}\big(I_{{\Gamma}_t}^{-1} (\hat{\theta}_{t-1}) I_{U_t}(\hat{\theta}_{t-1}) \big)$ and (b) follows from the results in \eqref{equ:SelfConcerdence1} and \eqref{equ:SelfConcerdence3}.

To bound the difference between the initialization and the true minimizer, we use triangle inequality and Jensen's inequality to get
\begin{equation}\label{equ:interstep_bound}
   \sqrt{\mathbb{E}\|\hat{\theta}_{t-1}-\theta_t^* \|_2^2}
  \le \sqrt{\mathbb{E} \|\hat{\theta}_{t-1}-\theta_{t-1}^*\|_2^2}+ \|\theta_t^*-\theta_{t-1}^*\|
  \le \sqrt{\mathbb{E} \|\hat{\theta}_{t-1}-\theta_{t-1}^*\|_2^2}+ \rho.
\end{equation}
From \eqref{equ:Delta_t}, we have
\begin{equation}
\mathbb{E}\|\hat{\theta}_{t-1}-\theta_{t-1}^*\|_2^2 \le \frac{2\varepsilon}{m},
\end{equation}
which yields
\begin{equation}
\mathbb{E}\|\hat{\theta}_{t-1}-\theta_t^* \|_2^2 \le \Big(\sqrt{\frac{2\varepsilon}{m}}+\rho\Big)^2 =\Delta_t^2.
\end{equation}

Thus, combining the above result with the bound in  \eqref{equ:UBThmStep2}, we can conclude that the following upper bound
\begin{equation}
\mathbb{E}[L_{U_t}(\hat{\theta}_t) - L_{U_t}(\theta_t^*)] \leq b(\acute{\tau}_t^{2}, \Delta_t ,K_t),
\end{equation}
holds with probability 1-$\delta$, where
\begin{equation}
\acute{\tau}_t^{2} = \Big(\frac{1+\beta_{t}}{1-\beta_{t}}\Big)^2 \frac{\mathrm{Tr}\big( I_{\Gamma_t^*}^{-1}(\theta_t^*) I_{U_t}(\theta_t^*) \big)}{2 \alpha_t}.
\end{equation}

This completes the proof of Theorem \ref{thm:BoundsWithKnownRho}.

%\subsection{Proof of the lower bound}
%To prove the lower-bound, since $I_{\bar{\Gamma}_t}(\theta_t^*) \succeq (1-\alpha_{t}) I_{U_t}(\theta_t^*)$, we can apply Lemma \ref{lem:BoundsWithArbitrarySamplingDistri} with $p=2$, then
%\begin{equation}
%\begin{split}
%\mathbb{E}[L_{U_t}(\hat{\theta}_t) - L_{U_t}(\theta_t^*)] \geq (1-\gamma_{t})\frac{{\mathrm{Tr}}\left( I_{\bar{\Gamma}_t}^{-1}(\theta_t^*) I_{U_t}(\theta_t^*) \right)}{K_t} - \frac{L_1^2}{(1-\alpha_t) K_t} \\
%\ge (1-\gamma_{t})\frac{{\mathrm{Tr}}\left( I_{\Gamma_t^*}^{-1}(\theta_t^*) I_{U_t}(\theta_t^*) \right)}{K_t} - \frac{L_1^2}{(1-\alpha_t) K_t}
%\end{split}
%\end{equation}
%holds.

\end{proof}

\section{Proof of Lemma \ref{lem:rho_upper_bound}}
\label{appx:lemma_2}
\begin{proof}[Proof of Lemma  \ref{lem:rho_upper_bound}]
The following inequalities hold from the strong convexity assumption and the fact that $\nabla L_{U_t}(\theta^*_{t})=\nabla L_{U_{t-1}}(\theta^*_{t-1})=0$:
\begin{align}
  L_{U_t}(\theta^*_{t-1}) & \ge L_{U_t}(\theta^*_{t})+\frac{1}{2}m\|\theta_t^*-\theta_{t-1}^*\|_2^2  \\
  L_{U_{t-1}}(\theta^*_{t}) & \ge L_{U_{t-1}}(\theta^*_{t-1})+\frac{1}{2}m\|\theta_t^*-\theta_{t-1}^*\|_2^2.
\end{align}
Then, adding and rearranging these inequalities yields
\begin{align}
   \frac{1}{m}&\Big[L_{U_t}(\theta^*_{t-1})- L_{U_t}(\theta^*_{t}) +  L_{U_{t-1}}(\theta^*_{t})-L_{U_{t-1}}(\theta^*_{t-1})\Big]\ge\|\theta_t^*-\theta_{t-1}^*\|_2^2.
\end{align}
\end{proof}

%For the case $\rho=\|\theta_t^*-\theta_{t-1}^*\|\ \forall t$
Moreover, we have the following relation
\begin{align}
\|\theta_t^*&-\theta_{t-1}^*\|_2^2 \nn \\
        &\le \frac{1}{m}\Big[L_{U_t}(\theta^*_{t-1})- L_{U_t}(\theta^*_{t})+  L_{U_{t-1}}(\theta^*_{t})-L_{U_{t-1}}(\theta^*_{t-1})\Big] \nn \\
        & = \frac{1}{m}\Big[\mathbb{E}_{X\sim U_t}\big[D\big(p(Y|X,\theta_{t}^*)\|p(Y|X,\theta_{t-1}^*)\big)\big]
        +\mathbb{E}_{X\sim U_{t-1}}\big[D\big(p(Y|X,\theta_{t-1}^*)\|p(Y|X,\theta_{t}^*)\big)\big]\Big],
\end{align}
where
\begin{equation}
  D(p\|q)\triangleq \int_{y \in \mathcal{Y}} p(y)\log \frac{p(y)}{q(y)}dy
\end{equation}
is the KL divergence between distribution $p$ and $q$.

Thus, an upper bound of $\rho$ can be constructed by estimating the symmetric KL divergence between $p(y|x,\theta_t^*)$ and $p(y|x,\theta_{t-1}^*)$ using the data pool $U_t$ and $U_{t-1}$, respectively.

\section{Proof of Theorem \ref{thm:rho_as}}
\label{appx:thm_2}
To analyze the performance of the estimator of $\rho$, we need to introduce a few results for sub-Gaussian random variables including the following key technical lemma from \cite{antonini2005note}. This lemma controls the concentration of sums of random variables that are sub-Gaussian conditioned on a particular filtration.

\begin{lem}\label{lem:sub_Gaussian}
Suppose we have a collection of random variables $\{V_i\}_{i=1}^n$ and a filtration
$\{\mathscr{F}_i\}_{i=0}^n$ such that for each random variable $V_i$ it holds that
\begin{enumerate}
  \item $\mathbb{E}[\exp\{s(V_i-\mathbb{E}[V_i| \mathscr{F}_{i-1}])\}| \mathscr{F}_{i-1}] \le e^{\frac{1}{2}\sigma_i^2s^2}$ with $\sigma_i^2$ a constant.
  \item $V_i$ is $\mathscr{F}_{i}$-measurable.
\end{enumerate}
\end{lem}
Then for every $\ba \in \mathbb{R}^n$ it holds that
\begin{equation*}
  \mathbb{P}\left\{\sum_{i=1}^n a_iV_i > \sum_{i=1}^n a_i \mathbb{E}[V_i| \mathscr{F}_{i-1}]+t \right\} \le \exp \Big\{\frac{t^2}{2\nu}\Big\}
\end{equation*}
with $\nu =\sum_{i=1}^n \sigma_i^2 a_i^2$. The other tail is similarly bounded.

If we can upper bound the conditional expectations $\mathbb{E}[V_i| \mathscr{F}_{i-1}] \le \xi_i$ by some constants $\xi_i$, then we have
\begin{equation}\label{equ:sub_gaussian_bound}
  \mathbb{P}\left\{\sum_{i=1}^n a_iV_i > \sum_{i=1}^n a_i \xi_i+t \right\} \le \exp \Big\{\frac{t^2}{2\nu}\Big\}.
\end{equation}

For our analysis, we generally cannot compute $\mathbb{E}[V_i| \mathscr{F}_{i-1}]$ directly, but we can find the upper bound $\xi_i$. To compute $\sigma_i^2$ for use in Lemma \ref{lem:sub_Gaussian}, we employ the following conditional version of Hoeffding's Lemma.

\begin{lem}\label{lem:Hoeffding}
(Conditional Hoeffding's Lemma): If a random
variable $V$ and a sigma algebra $\mathscr{F}$ satisfy $a \le V \le b$ and
$E[V|\mathscr{F}] = 0$, then
\begin{equation*}
  \mathbb{E}[e^{sV}|\mathscr{F}]\le \exp\Big\{\frac{1}{8}(b-a)^2 s^2 \Big\}.
\end{equation*}
\end{lem}

\begin{proof}[Proof of Lemma Theorem \ref{thm:rho_as}]
To simplify our proof, we look at a special case where $\|\theta_{t}^*-\theta_{t-1}^*\| =\rho$ holds. The proof for the case $\|\theta_{t}^*-\theta_{t-1}^*\| \le \rho$ is similar, and more details about the window function $h_W$ can be found in \cite{wilson2018adaptive}.

For the case $\|\theta_{t}^*-\theta_{t-1}^*\| =\rho$, we use the following estimator to combine the one-step estimator $\widetilde{\rho}_t$
\begin{equation}
  \acute{\rho}_t^2 = \frac{1}{t-1} \sum_{i=2}^t \widetilde{\rho}_i^2 =\frac{1}{m(t-1)} \sum_{i=2}^t\big(\hat{L}_{U_i}(\hat{\theta}_{i-1}) - \hat{L}_{U_i}(\hat{\theta}_{i})+\hat{L}_{U_{i-1}}(\hat{\theta}_{i})-\hat{L}_{U_{i-1}}(\hat{\theta}_{i-1})\big).
\end{equation}

% To proceed, denote
%\begin{equation}
%  \widetilde{\rho}_i^2 = \frac{1}{m}\big(\hat{L}_{U_i}(\hat{\theta}_{i-1}) - \hat{L}_{U_i}(\hat{\theta}_{i})+\hat{L}_{U_{i-1}}(\hat{\theta}_{i})-\hat{L}_{U_{i-1}}(\hat{\theta}_{i-1})\big)
%\end{equation}

We denote
\begin{equation}
  \rho_t^2 \triangleq \frac{1}{m(t-1)} \sum_{i=2}^t\big(L_{U_i}(\theta_{i-1}^*) - L_{U_i}(\theta_{i}^*)+L_{U_{i-1}}(\theta_{i}^*)-L_{U_{i-1}}(\theta_{i-1}^*)\big)\ge \rho^2.
\end{equation}
where the inequality follows from Lemma \ref{lem:rho_upper_bound}. We want to construct $\hat{\rho}_t$, such that $\hat{\rho}_t^2 \ge \rho_t^2 \ge \rho^2$ almost surely. Then, we have
\begin{align}
  \rho_t^2 - \acute{\rho}_t^2  = & \frac{1}{m(t-1)} \Big(\sum_{i=2}^t L_{U_i}(\theta_{i-1}^*) - \hat{L}_{U_i}( \hat{\theta}_{i-1})  + \sum_{i=2}^t L_{U_{i-1}}(\theta_{i}^*) - \hat{L}_{U_{i-1}}(\hat{\theta}_{i}) \\
   & \qquad + \hat{L}_{U_1}(\hat{\theta}_{1}) - L_{U_1}(\theta_{1}^*)+2\sum_{i=2}^{t-1} \hat{L}_{U_i}(\hat{\theta}_{i}) - L_{U_i}(\theta_{i}^*)
   +\hat{L}_{U_t}(\hat{\theta}_{t}) - L_{U_t}(\theta_{t}^*) \Big).
\end{align}
Define
\begin{align}
  U_t &\triangleq \frac{1}{(t-1)}\sum_{i=2}^t \frac{1}{m}\big(L_{U_i}(\theta_{i-1}^*) - \hat{L}_{U_i}( \hat{\theta}_{i-1}) \big), \\
  V_t &\triangleq \frac{1}{(t-1)}\sum_{i=2}^t \frac{1}{m}\big(L_{U_{i-1}}(\theta_{i}^*) - \hat{L}_{U_{i-1}}(\hat{\theta}_{i}) \big), \\
  W_t & \triangleq \frac{1}{m(t-1)} \Big(\hat{L}_{U_1}(\hat{\theta}_{1}) - L_{U_1}(\theta_{1}^*)+2\sum_{i=2}^{t-1} \big(\hat{L}_{U_i}(\hat{\theta}_{i}) - L_{U_i}(\theta_{i}^*)\big)
   +\hat{L}_{U_t}(\hat{\theta}_{t}) - L_{U_t}(\theta_{t}^*) \Big).
\end{align}
Then it holds that
\begin{equation}
  \rho_t^2 - \acute{\rho}_t^2 = U_t + V_t + W_t.
\end{equation}
Now, we look at bounding $\mathbb{E}\big[L_{U_i}(\theta_{i-1}^*)-\hat{L}_{U_i}(\hat{\theta}_{i-1})\big]$,
$\mathbb{E}\big[L_{U_{i-1}}(\theta_{i}^*) - \hat{L}_{U_{i-1}}(\hat{\theta}_{i})\big]$ and $\mathbb{E}\big[\hat{L}_{U_i}(\hat{\theta}_{i}) - L_{U_i}(\theta_{i}^*)\big]$ in $U_t$, $V_t$ and $W_t$, respectively.

Note that, the samples at time step $i-1$ are independent with samples at time $i$, hence,
\begin{align}\label{equ:U_t_expectation}
  \mathbb{E}\big[\hat{L}_{U_i}(\hat{\theta}_{i-1})\big] & =  \mathbb{E}\Big[ \mathbb{E}_{X_{k,i}\sim \bar{\Gamma}_i, Y_{k,i}\sim p(Y|X_{k,i},\theta_i^*)}\big[\frac{1}{K_i }\sum_{i=1}^{K_i} \frac{\ell(Y_{k,i}|X_{k,i},\hat{\theta}_{i-1})}{N_i\bar{\Gamma}_i(X_{k,i})} \big]\Big]\nn\\
   & = \mathbb{E}\Big[ \mathbb{E}_{X_i\sim U_i, Y_i\sim p(Y|X_i,\theta_i^*)}\big[ \ell(Y_{t}|X_{t},\hat{\theta}_{i-1})\big]\Big]\nn\\
  & = \mathbb{E}\big[ {L}_{U_i}(\hat{\theta}_{i-1}) ].
\end{align}
Thus,
\begin{align}
  \mathbb{E}\big[L_{U_i}(\theta_{i-1}^*)-\hat{L}_{U_i}(\hat{\theta}_{i-1})\big] &= \mathbb{E}\big[L_{U_i}(\theta_{i-1}^*)-{L}_{U_i}(\hat{\theta}_{i-1})\big], \\
  \mathbb{E}\big[L_{U_{i-1}}(\theta_{i}^*) - \hat{L}_{U_{i-1}}(\hat{\theta}_{i})\big] &= \mathbb{E}\big[L_{U_{i-1}}(\theta_{i}^*) - {L}_{U_{i-1}}(\hat{\theta}_{i})\big].
\end{align}
%\frac{1}{2}{\mathrm{Tr}}\big( I_{\Gamma_t}^{-1}(\theta_t^*) I_{U_t}(\theta_t^*) \big)
We use Lemma \ref{lem:OriginalConvergence} to construct bounds for these two terms.
Let
\begin{equation}
  Q(\theta)=\big(L_{U_i}(\theta_{i-1}^*)- L_{U_i}(\theta) \big)^2, \mbox{ and } \ \psi_k (\theta)=\ell(Y_k|X_k,\theta), \quad 1\le k \le K_{i-1},
\end{equation}
where $X_k \sim \bar{\Gamma}_{i-1}$ and $Y_k \sim p(Y|X_k,\theta_{i-1}^*)$. It can be verified that
\begin{equation}
  \hat{\theta}_{i-1} = \argmin_{\theta \in \Theta} \sum_{k}^{K_{i-1}}\psi_k(\theta),
  \quad \theta^* = \argmin_{\theta \in \Theta} P(\theta) = \argmin_{\theta \in \Theta} \mathbb{E} [\psi(\theta)] =\theta_{i-1}^*,
\end{equation}
and $\nabla Q(\theta_{i-1}^*)=0$. All the conditions in Lemma \ref{lem:OriginalConvergence} are satisfied. We have
\begin{equation}
  \nabla^2 P(\theta^*) = I_{\bar{\Gamma}_{i-1}}(\theta_{i-1}^*), \quad \nabla^2 Q(\theta^*) = 2I_{U_i}(\theta_{i-1}^*).
\end{equation}
Thus,
\begin{align}
 \big(\mathbb{E}\big[L_{U_i}&(\theta_{i-1}^*)-{L}_{U_i}(\hat{\theta}_{i-1})\big] \big)^2 \nn \\
 & \le \mathbb{E}\big[\big(L_{U_i}(\theta_{i-1}^*)-{L}_{U_i}(\hat{\theta}_{i-1})\big)^2 \big] \nn\\
 %&\le \mathbb{E}[\big(L_{U_i}(\hat\theta_{i-1}) - L_{U_i}(\theta_{i-1}^*) \big)^2]\nn \\
 &\le (1+\gamma_{i-1}) \frac{\mathrm{Tr}\big(I_{\bar{\Gamma}_{i-1}}(\theta_{i-1}^*)^{-1}I_{U_i}(\theta_{i-1}^*)\big) }{K_{i-1}} + \frac{\max_{\theta\in \Theta} \left[L_{U_i}(\theta)-L_{U_i}(\theta_{i-1}^*)\right]^2}{K_{i-1}^2} \nn \\
 %& \le b\Big(\mathrm{Tr}\big( I_{\bar{\Gamma}_i}^{-1}(\theta_i^*) I_{U_i}(\theta_i^*) \big), \sqrt{{2\varepsilon}/{m}}+\rho, K_i \Big)
 &\triangleq A_i.
\end{align}
Similarly, we have
\begin{align}
 \big(\mathbb{E}[L_{U_{i-1}}&(\theta_{i}^*) - {L}_{U_{i-1}}(\hat{\theta}_{i})] \big)^2 \nn \\
 %& \le \mathbb{E}\big[\big(L_{U_{i-1}}(\theta_{i}^*) - {L}_{U_{i-1}}(\hat{\theta}_{i})\big)^2 \big] \nn\\
 &\le (1+\gamma_{i}) \frac{\mathrm{Tr}\big(I_{\bar{\Gamma}_{i}}(\theta_{i}^*)^{-1}I_{U_{i-1}}(\theta_{i}^*)\big) }{K_{i}} + \frac{\max_{\theta\in \Theta} \left[L_{U_{i-1}}(\theta)-L_{U_{i-1}}(\theta_{i}^*)\right]^2}{K_{i}^2} \nn \\
 %& \le b\Big(\mathrm{Tr}\big( I_{\bar{\Gamma}_i}^{-1}(\theta_i^*) I_{U_i}(\theta_i^*) \big), \sqrt{{2\varepsilon}/{m}}+\rho, K_i \Big)
 &\triangleq B_i.
\end{align}
For the term $\mathbb{E}\big[L_{U_i}(\theta_{i}^*) - \hat{L}_{U_i}(\hat{\theta}_{i})\big]$  in $W_t$, suppose that the samples used to estimate $\hat{\theta}_{i}$ and the samples used to compute $\hat{L}_{U_i}$ are independent. This can be done by splitting the samples at each time step $i$. Note that this assumption is just required to proceed with the theoretical analysis; we will use all the samples to estimate $\hat{\theta}_{i}$ in practice.

Then, similar argument holds as in \eqref{equ:U_t_expectation}, and we have
\begin{equation}
  \mathbb{E}\big[\hat{L}_{U_i}(\hat{\theta}_{i}) - L_{U_i}(\theta_{i}^*)\big] = \mathbb{E}\big[{L}_{U_i}(\hat{\theta}_{i}) - L_{U_i}(\theta_{i}^*)\big]\ge 0.
\end{equation}
where the inequality follows from the fact that $\theta_t^*$ is the minimizer of $L_{U_t}(\theta)$. Applying the upper bound in Lemma \ref{lem:BoundsWithArbitrarySamplingDistri}, this term can be bounded as
\begin{align}
  0 \le \mathbb{E}\big[L_{U_i}(\theta_{i}^*) - {L}_{U_i}(\hat{\theta}_{i}) \big] &\le  (1+\gamma_{i})\frac{\mathrm{Tr}\big( I_{\bar{\Gamma}_i}^{-1}(\theta_i^*) I_{U_i}(\theta_i^*)}{2K_i} + \frac{\max_{\theta\in \Theta} \left[L_{U_{i}}(\theta)-L_{U_{i}}(\theta_{i}^*)\right]}{K_{i}^2} \nn\\
  %&\le b\Big(\mathrm{Tr}\big( I_{\bar{\Gamma}_i}^{-1}(\theta_i^*) I_{U_i}(\theta_i^*) \big), \sqrt{{2\varepsilon}/{m}}+\rho, K_i \Big)\nn \\
  & \triangleq C_i.
\end{align}
The resulting bounds on the expectation of $U_t$, $V_t$, and $W_t$
denoted $\bar{U}_t$, $\bar{V}_t$, and $\bar{W}_t$ are as follows:
\begin{align}
  \bar{U}_t & = \frac{1}{m(t-1)}\sum_{i=2}^t \sqrt{A_i},\\
  \bar{V}_t & = \frac{1}{m(t-1)}\sum_{i=2}^t \sqrt{B_i},\\
  \bar{W}_t & = \frac{1}{m(t-1)}(C_1+2\sum_{i=2}^{t-1} C_i+C_t).
\end{align}
Now, we find the upper bound $\xi_i$ to upper bound the expectation as we mentioned in \eqref{equ:sub_gaussian_bound}. Then it holds that
\begin{align}
  \mathbb{P}\Big\{\rho_t^2 - \acute{\rho}_t^2 &> \bar{U}_t +\bar{V}_t+\bar{W}_t +r_t\Big\} \nn \\
   & =   \mathbb{P}\Big\{U_t+V_t+W_t > \bar{U}_t +\bar{V}_t+\bar{W}_t +r_t\Big\} \nn \\
   & \le  \mathbb{P}\Big\{U_t> \bar{U}_t + \frac{1}{3}r_t\Big\}+
    \mathbb{P}\Big\{V_t> \bar{V}_t + \frac{1}{3}r_t\Big\}+
     \mathbb{P}\Big\{W_t> \bar{W}_t + \frac{1}{3}r_t\Big\}.
\end{align}

To bound these probabilities with \eqref{equ:sub_gaussian_bound}, we first bound the moment generating functions using Lemma \ref{lem:Hoeffding},
\begin{equation}
  \frac{1}{m}|\hat{L}_{U_i}(\hat{\theta}_{i}) - L_{U_i}(\theta_{i}^*)| \le \frac{L_b}{2m}  \max_{\theta \in \Theta}\|\theta-\theta_i^*\|^2 \le \frac{L_b}{2m} \mathrm{Diameter}(\Theta)^2,
\end{equation}
and
\begin{align}
    \frac{1}{m}|L_{U_i}(\theta_{i-1}^*)-\hat{L}_{U_i}(\hat{\theta}_{i-1})|
  &\le \frac{1}{m}|L_{U_i}(\theta_{i-1}^*)-L_{U_i}(\theta_{i}^*) |+ \frac{1}{m}|L_{U_i}(\theta_{i}^*)-\hat{L}_{U_i}(\hat{\theta}_{i-1})| \nn \\
 &\le \frac{L_b}{m} \mathrm{Diameter}(\Theta)^2.
\end{align}
Then, we apply Lemma \ref{lem:sub_Gaussian} and Lemma \ref{lem:Hoeffding} with $\sigma_i^2 = \frac{L_b^2}{4m^2}  \mathrm{Diameter}^4(\Theta) $ for the terms in $U_t$ and $V_t$, and apply $\sigma_i^2 = \frac{L_b^2}{16m^2}  \mathrm{Diameter}^4(\Theta) $ for the terms in $W_t$, respectively. We have
\begin{align}
  \nu_U=\nu_V & = \frac{L_b^2}{4 m^2}  \mathrm{Diameter}(\Theta)^4 \sum_{i=2}^t \frac{1}{(t-1)^2} = \frac{L_b^2}{4(t-1)m^2}\mathrm{Diameter}(\Theta)^4, \\
  \nu_W & \le \frac{L_b^2}{16m^2}  \mathrm{Diameter}(\Theta)^4 \sum_{i=2}^t \big(\frac{2}{t-1}\big)^2 = \frac{L_b^2}{4(t-1)m^2}\mathrm{Diameter}(\Theta)^4.
\end{align}
Let $D_t \triangleq \bar{U}_t +\bar{V}_t+\bar{W}_t$. Then we obtain
\begin{align}
  \mathbb{P}\Big\{\rho_t^2> \acute{\rho}_t^2 + D_t +r_t\Big\}
   & \le   %2\exp \Big\{-\frac{m(t-1)r_t^2}{18L_b \mathrm{Diameter}^2(\Theta)} \Big\}
   3 \exp \Big\{-\frac{2m^2(t-1)r_t^2}{9L_b^2 \mathrm{Diameter}^4(\Theta)} \Big\}.
\end{align}
Then it follows the assumption in Theorem \ref{thm:rho_as} that
\begin{equation}
  \sum_{t=2}^ \infty \mathbb{P}\Big\{ \acute{\rho}_t^2 + D_t +r_t < \rho_t^2\Big\}
  \le \sum_{t=2}^ \infty  %2\exp \Big\{-\frac{m(t-1)r_t^2}{18L_b \mathrm{Diameter}^2(\Theta)} \Big\}
   3 \exp \Big\{-\frac{2m^2(t-1)r_t^2}{9L_b^2 \mathrm{Diameter}^4(\Theta)} \Big\} < \infty.
\end{equation}
Therefore, by the Borel-Cantelli Lemma, for all $t$ large enough it holds that
\begin{equation}
  \hat{\rho}_t^2 = \acute{\rho}_t^2 + D_t +r_t \ge \rho_t^2
\end{equation}
almost surely. Finally, it holds that $\rho_t^2 \ge \rho^2 $ from Lemma \ref{lem:rho_upper_bound}, which proves the result.
\end{proof}

\section{Proof of Theorem \ref{thm:mean_tracking}}
To prove Theorem \ref{thm:mean_tracking}, we use the following result from Theorem 3 in \cite{wilson2018adaptive}.

\begin{lem}
If  $\hat{\rho}_t \ge \rho$ almost surely for $t$ sufficiently large, then with
\begin{equation}
K_t \ge K_t^* \triangleq \min\Big\{K\ge 1 \Big| b\Big(d/2, \big(\sqrt{\frac{2\varepsilon}{m}}+\hat{\rho}_{t-1}\big)^2, K \Big) \le \varepsilon \Big\}
\end{equation}
samples, we have $\limsup_{t\to \infty} (\mathbb{E}[L_{U_t}(\hat{\theta}_t)] - L_{U_t}({\theta}_t^*))\le \varepsilon$ almost surely.
\end{lem}

\begin{proof}[Proof of Theorem \ref{thm:mean_tracking}]
From Theorem \ref{thm:rho_as}, we know that the proposed estimate $\hat{\rho}_t^2 \ge \rho^2$ almost surely, which implies $\hat{\rho}_t \ge \rho$ almost surely. Directly applying the above lemma completes the proof.
\end{proof}

\section{Estimation of $m$ and $L_{b}$}
\label{appx:estimation}

\begin{figure}
\centering     %%% not \center
\subfigure[]{\label{fig:RegEststrCvx}\includegraphics[width=0.45\textwidth]{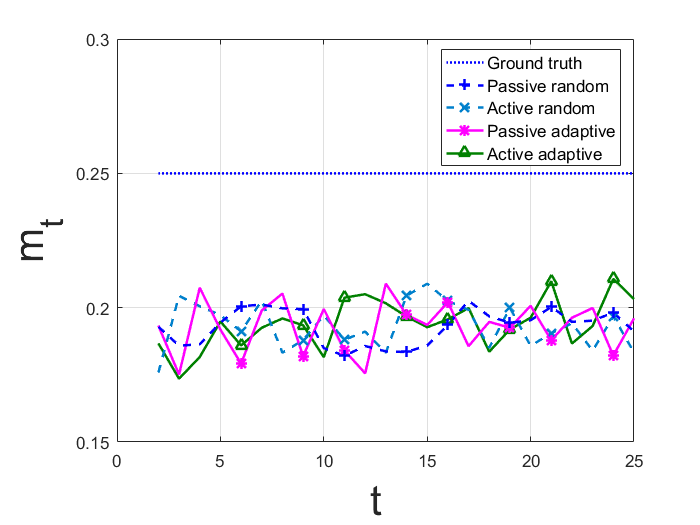}}
\subfigure[]{\label{fig:RegEstBoundedTrace}\includegraphics[width=0.45\textwidth]{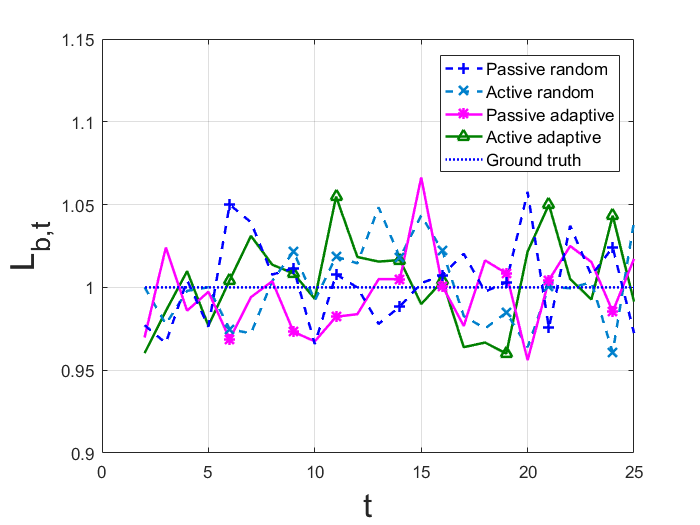}}
\caption{Estimated parameter on the regression task over synthetic data. (a) Estimated strongly convex parameter. (b) Estimated largest eigenvalue.}
\end{figure}

We construct the estimator of $m$ and $L_b$ with the samples drawn from distribution $\bar{\Gamma}_t$. %It can be proved that our estimator can produce a safe and conservative estimation of $\rho_t$ and $K_t^*$.
By the assumption of strongly convexity, we have
\begin{equation}
L_{U_t}(\theta) \geq L_{U_t}(\theta') + \langle
\nabla L_{U_t}(\theta'), \theta-\theta' \rangle + \frac{m}{2} \| \theta-\theta' \|^2, ~\forall \theta, \theta' \in \Theta,
\end{equation}
which implies that
\begin{equation}\label{equ:mLeq}
m \leq \frac{ L_{U_t}(\theta) - L_{U_t}(\theta') - \langle
\nabla L_{U_t}(\theta'), \theta-\theta' \rangle }{\frac{1}{2} \| \theta-\theta' \|^2}
\end{equation}
holds for any $\theta, \theta' \in \Theta$.

%According to \eqref{equ:rho_estimator}, we need the local strongly convexity of $L_{U_{t-1}}$ and $L_{U_t}$ between $\hat{\theta}_{t-1}$ and $\hat{\theta}_{t}$ instead of the global strongly convexity for $\forall~\theta, \theta^* \in \Theta$.

Since $m$ is the smallest value satisfying \eqref{equ:mLeq} for any $\theta, \theta' \in \Theta$, we consider following estimator
\begin{equation}\label{equ:hatM}
\widetilde{m}_t \triangleq \min_{\theta, \theta' \in \Theta_t}\frac{2}{K_t}\sum_{k=1}^{K_t}\frac{  \ell (Y_{k,t}|X_{k,t},\theta) - \ell (Y_{k,t}|X_{k,t},\theta') - \langle \nabla \ell (Y_{k,t}|X_{k,t},\theta'), \theta-\theta' \rangle }{N_t  \bar{\Gamma}_{t}(X_{k,t})\| \theta-\theta'\|^2}.
\end{equation}

Following \eqref{equ:hatM}, we have
\begin{equation}\label{equ:MconservativeEst}
\begin{split}
\mathbb{E}( \widetilde{m}_t ) &= \mathbb{E}_{X_{k,t}\sim \Gamma_t}\left\{ \min_{\theta, \theta' \in \Theta_t}\frac{2}{K_t}\sum_{k=1}^{K_t}\frac{  \ell (Y_{k,t}|X_{k,t},\theta) - \ell (Y_{k,t}|X_{k,t},\theta') - \langle \nabla \ell (Y_{k,t}|X_{k,t},\theta'), \theta-\theta' \rangle }{N_t  \Gamma_{t}(X_{k,t})\| \theta-\theta'\|^2} \right\}\\
&\leq \min_{\theta, \theta' \in \Theta_t} \mathbb{E}_{X_{k,t}\sim \Gamma_t}\left\{ \frac{2}{K_t}\sum_{k=1}^{K_t}\frac{  \ell (Y_{k,t}|X_{k,t},\theta) - \ell (Y_{k,t}|X_{k,t},\theta') - \langle \nabla \ell (Y_{k,t}|X_{k,t},\theta'), \theta-\theta' \rangle }{N_t  \Gamma_{t}(X_{k,t})\| \theta-\theta'\|^2} \right\}\\
&=\min_{\theta, \theta' \in \Theta_t} \mathbb{E}_{X_{k,t}\sim U_t}\left\{\frac{2}{K_t}\sum_{k=1}^{K_t}\frac{  \ell (Y_{k,t}|X_{k,t},\theta) - \ell (Y_{k,t}|X_{k,t},\theta') - \langle \nabla \ell (Y_{k,t}|X_{k,t},\theta'), \theta-\theta' \rangle }{\| \theta-\theta'\|^2} \right\}\\
&=\min_{\theta, \theta' \in \Theta_t} \frac{ L_{U_t}(\theta) - L_{U_t}(\theta') - \langle
\nabla L_{U_t}(\theta'), \theta-\theta' \rangle }{\frac{1}{2} \| \theta-\theta' \|^2}\\
&=m,
\end{split}
\end{equation}
which implies that $\widetilde{m}_t$ is a conservative estimate of $m$. In practice, the strongly convex parameter $m$ may also vary with time $t$. Thus, we use the following estimator to combine the one-step estimator $\widetilde {m}_{t}$,
\begin{equation}
  \hat{m}_t = \min\{\widetilde{m}_{t-1}, \widetilde{m}_t \},
\end{equation}
for $t\ge2$.

Moreover, following the boundedness assumption in Assumption \ref{assump:regularity}, we have
\begin{equation}
\max_{\theta\in\Theta} ~\lambda_{\max}~\left[ I_{U_t}(\theta)\right] \leq L_{{b}},
\end{equation}
where $\lambda_{\max}(\cdot)$ denotes the maximal eigenvalue of a square matrix. In this case, we consider following estimator
\begin{equation}
\hat{L}_{b,t} \triangleq \max_{\theta_t\in\Theta_t} ~ \lambda_{\max}~\left[ \frac{1}{K_t}\sum_{k=1}^{K_t} \frac{1}{N_t}\frac{1}{\Gamma_{t}(X_{k,t})} H(X_{k,t}, \theta_t)\right].
\end{equation}
Similarly, $\hat{L}_{b}$ is also a conservative estimate of $L_{\mathrm{b}}$. That is,
\begin{equation}\label{equ:EstLb}
\begin{split}
\mathbb{E} (\hat{L}_{b,t}) &= \mathbb{E}_{X_{k,t}\sim \Gamma_t} \left\{ \max_{\theta_t\in\Theta_t} ~\lambda_{\max}~\left(  \frac{1}{K_t}\sum_{k=1}^{K_t} \frac{1}{N_t}\frac{1}{\Gamma_{t}(X_{k,t})} H(X_{k,t}, \theta_t) \right) \right\}\\
&\geq \max_{\theta_t\in\Theta_t}~\mathbb{E}_{X_{k,t}\sim \Gamma_t}\left\{\lambda_{\max}~\left[  \frac{1}{K_t}\sum_{k=1}^{K_t} \frac{1}{N_t} \frac{1}{\Gamma_{t}(X_{k,t})} H(X_{k,t}, \theta_t) \right] \right\} \\
&\ge \max_{\theta_t\in\Theta_t}~\lambda_{\max}~\left[ I_{U_t}(\theta_t)\right]\\
&=L_{b}.
\end{split}
\end{equation}

Fig. \ref{fig:RegEststrCvx} and Fig. \ref{fig:RegEstBoundedTrace} demonstrate our estimation of $\hat{m}_t$ and $\hat{L}_{b,t}$ in the synthetic regression problem described in Section \ref{Sec:experiments}, respectively. %Our estimator produces a conservative estimate of $m$. However, since in regression task the Hessian matrix is not related with $\theta$, equality holds in the inequality shown in \eqref{equ:EstLb}, and hence the estimator produces an unbiased estimate of $L_{b,t}$.

%
%\subsection{Bounds for SGD}
%In practice, the ERM solution in \eqref{equ:MLE} cannot be solved accurately, we can only apply optimization algorithm such as SGD to find an approximate minimizer.
%Thus, we need the following assumptions on the optimization algorithm to conduct our adaptive sequential learning:
%
%\begin{thm}
%\begin{equation}\label{equ:UpperBoundOfKnownThm}
%\mathbb{E}[L_{U_t}(\hat{\theta}_t) - L_{U_t}(\theta_t^*)] \leq  \frac{2L_b\mathrm{Tr}\Big( I_{\Gamma_t^*}^{-1}(\theta_t^*) I_{U_t}(\theta_t^*) \Big)}{mK_t}+ \frac{2 L_b^2\mathbb{E}[L_{U_t}(\hat{\theta}_{t-1}) - L_{U_t}(\theta_t^*)]}{m^2K_t^2}.
%\end{equation}
%\end{thm}

%  & \le \frac{1}{K_t}\mathbb{E}\bigg[ \mathbb{E}_{X_t\sim \Gamma_t, Y_t\sim P(Y|X,\theta_t^*)} \Big[\frac{1}{N_t}\frac{1}{\Gamma_t(X_t)}\big(\log\frac
%  {p(Y_{t}|X_{t},\hat{\theta}_{t})}{p(Y_{t}|X_{t},\hat{\theta}_{t-1})}
%  -\log\frac
%  {p(Y_{t}|X_{t},\theta_{t}^*)}{p(Y_{t}|X_{t},\theta_{t-1}^*)}\big) \Big]^2\bigg]\nn \\
%  & \le \frac{1}{K_t} \mathbb{E}\bigg[ \mathbb{E}_{X_t\sim U_t, Y_t\sim P(Y|X,\theta_t^*)}\Big[ \frac{1}{N_t}\frac{1}{\Gamma_t(X_t)} \log\frac
%  {p(Y_{t}|X_{t},\hat{\theta}_{t})}{p(Y_{t}|X_{t},\hat{\theta}_{t-1})}\Big]\bigg]

\section{Experiments on Synthetic Classification} %and Spam Detection

%\textbf{Synthetic classification:}

We consider solving a sequence of binary classification problems by using logistic regression. At time $t$, the features of two classes are drawn from Gaussian distribution with different means $\mu_{1,t}$ and $-\mu_{1,t}$. More specifically, the features are 2-dimensional Gaussian vectors with $\|\mu_{1,t}\|_2=2$ and variance $0.25 I$. The parameter $\theta_t$ is learned by minimizing the following log-likelihood function
\begin{equation}\label{equ:LogLikeliLogisticRegression}
\ell(y_{k,t}|x_{k,t},\theta_t) = \log(1 + \exp^{-y_{k,t}\theta_t^\top x_{k,t}}).
\end{equation}
To ensure the change of minimizers is bounded, we set that $\mu_{1,t}$ is drifting with a constant rate along a $2$-dimensional sphere. We further set $\rho=0.1$ and $\epsilon=0.5$.

%Moreover, the change rate of true minimizers is set to be $\rho = 0.1$ and the target excess risk is $\varepsilon=0.5$.

\begin{figure}[htp!]
\centering     %%% not \center
\subfigure[]{\label{fig:SynClassTestloss}\includegraphics[width=0.32\textwidth]{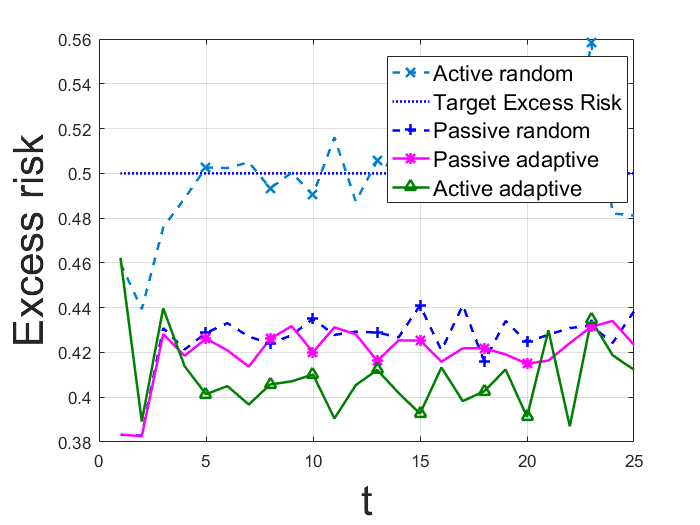}}
\subfigure[]{\label{fig:SynClaRhoEst}\includegraphics[width=0.32\textwidth]{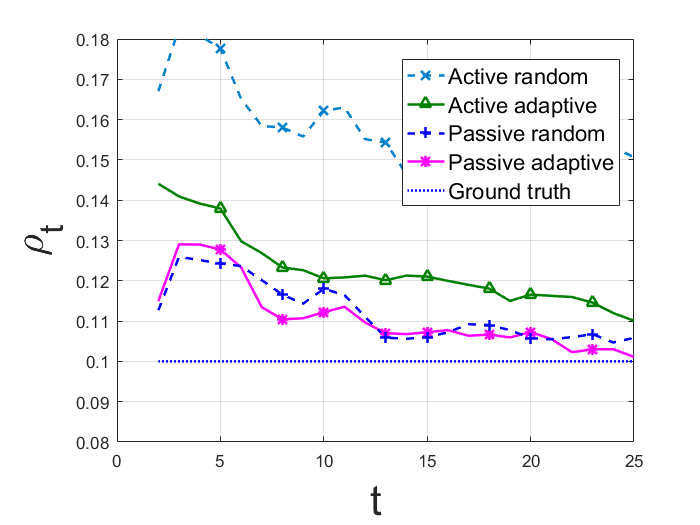}}
\subfigure[]{\label{fig:SynClassError}\includegraphics[width=0.32\textwidth]{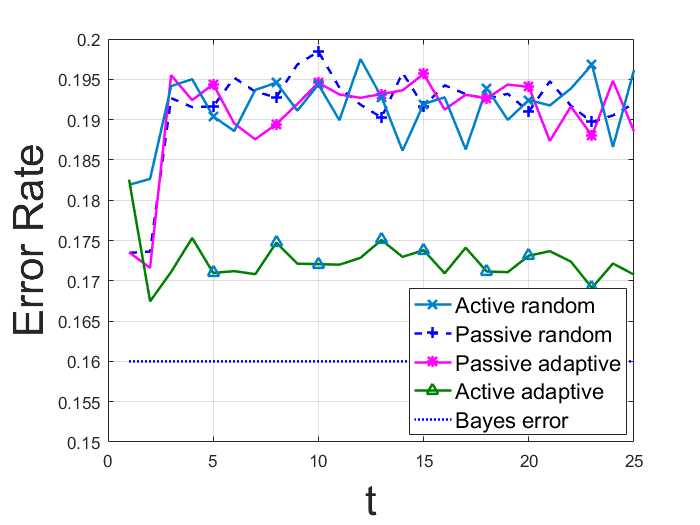}}
%\subfigure[]{\label{fig:SpamClassError}\includegraphics[width=0.32\textwidth]{SpamTestLoss.png}}
%\subfigure[]{\label{fig:SpamClassRhoEst}\includegraphics[width=0.32\textwidth]{SpamRhoEst.png}}
%\subfigure[]{\label{fig:SpamClassError}\includegraphics[width=0.32\textwidth]{SpamErrPer.png}}
\caption{Experiments on synthetic classification: (a) Excess risk. (b) Estimated rate of change of minimizers. (c) Classification error. }\label{fig:SynClassAndSpamRes}
\end{figure}

Fig. \ref{fig:SynClassAndSpamRes} shows that active adaptive learning outperforms other baseline algorithms in the synthetic classification problem.

%Experiments on Spam detection: (d) Excess risk. (e) Estimated rate of change of minimizers. (f) Classification error.

%\textbf{Spam detection:} We use the ECML/PKDD Discovery Challenge 2006 dataset \footnote{http://www.ecmlpkdd2006.org/challenge.html} and consider the spam detection problem using logistic regression model. In this dataset, the features of E-mails are the frequencies  of the keywords in the E-mail. Since the dictionary size of this dataset is of about 150,000 words, we first reduce the dimension of these features by selecting $10$ keywords which are the most different between the spam and normal E-mails in terms of the frequency. Then, we train the spam classifier of each user by minimizing the loss function in \eqref{equ:LogLikeliLogisticRegression}.
%
%The spam classifier for each user is different but related, and hence we assume that the changes in the minimizers are bounded. We take the emails of 8 different users from the dataset as the unlabeled data pool of the sequence of machine learning problems, which need to be solved in our active and adaptive learning framework. In this case, we apply our algorithms and sequentially learn the spam classifier under target excess risk $\varepsilon=1$.
%

%Fig. \ref{fig:RegFigures} and Fig. \ref{fig:SynClassAndSpamRes} show that the empirical excess risk of the passive and adaptive algorithm can be lower than the target excess risk, implying that our bound